\documentclass[twoside,11pt]{article}

\usepackage[margin=1in]{geometry}

\usepackage{graphicx,times,authblk}
\usepackage{amssymb,amsthm,natbib}
\usepackage{array}       
\usepackage{color}
\usepackage{chngpage}

\usepackage{tablefootnote,adjustbox}
\usepackage{amsmath,bbm}
\usepackage{subcaption}
\usepackage{ifthen}
\usepackage{xspace}
\usepackage{algorithm}

\usepackage[noend]{algorithmic}
\usepackage{hyperref}       
\allowdisplaybreaks
\definecolor{darkblue}{rgb}{0,0,0.4}
\hypersetup{
    colorlinks,
    citecolor=blue,
    linkcolor=red,
    urlcolor=darkblue
}

\newcommand{\red}[1]{{\color{black} #1}}

\newcommand{\iprod}[2]{\langle #1, #2 \rangle}
\newcommand{\E}[1]{\mathbb{E}\left[#1\right]}
\newcommand{\trans}[1]{{#1}^{\top}}
\newcommand{\inv}[1]{{#1}^{-1}}

\newcommand{\defeq}{\stackrel{\mathrm{def}}{=}}

\newcommand{\R}{\mathbb{R}}
\newcommand{\Rd}{\mathbb{R}^{d}}

\newcommand{\SRd}{\mathcal{S}(d)}
\newcommand{\SplRd}{\mathcal{S}^{+}(\Rd)}

\newcommand{\gammabmaxdiv}{\gamma_{b,\textrm{max}}^{div}}
\newcommand{\gammabmax}[1][b]{\gamma_{#1,\textrm{max}}}
\newcommand{\trace}[1]{\Tr\left(#1\right)}
\newcommand{\order}[1]{\mathcal{O}\left(#1\right)}
\newcommand{\bthresh}{b_{\textrm{thresh}}}

\newcommand{\norm}[1]{\left\| #1 \right\|}
\newcommand{\twonorm}[1]{\norm{#1}_2}

\newcommand{\D}{\mathcal{D}}
\newcommand{\minibatchSGD}{\textrm{Minibatch-TailAveraging-SGD}}
\newcommand{\DOminibatchSGD}{\textrm{MinibatchDoublingPartialAveragingSGD}}
\newcommand{\cnH}{\kappa}
\newcommand{\bl}{b_{\ell}}
\newcommand{\rhom}{\rho_{\text{m}}}
\newcommand{\phiv}{\boldsymbol{\Phi}}

\renewcommand{\vec}[1]{\mathbf{#1}}
\newcommand{\w}{\vec{w}}
\newcommand{\wbar}{\overline{\vec{w}}}
\newcommand{\ws}{\vec{w^*}}
\newcommand{\x}{\vec{x}}
\newcommand{\etav}{\boldsymbol{\eta}}
\newcommand{\zetav}{\boldsymbol{\zeta}}
\newcommand{\xiv}{\boldsymbol{\xi}}
\newcommand{\etavb}{{\bar{\etav}}}
\newcommand{\gb}{\frac{\gamma}{b}}
\newcommand{\xTrans}{\trans{\vec{x}}}

\newcommand{\mat}[1]{\mathbf{#1}}
\newcommand{\A}{\mat{A}}
\newcommand{\B}{\mat{B}}
\newcommand{\PP}{\mat{P}}
\newcommand{\Q}{\mat{Q}}

\newcommand{\U}{\mat{U}}

\newcommand{\W}{\mat{W}}
\renewcommand{\H}{\mat{H}}
\newcommand{\Hinv}{\inv{\H}}
\newcommand{\eye}{\mat{I}}
\newcommand{\Sig}{\mat{\Sigma}}

\newcommand{\Tone}{\mathfrak{T}_1}
\newcommand{\Ttwo}{\mathfrak{T}_2}

\newcommand{\tensor}[1]{\mathcal{#1}}
\newcommand{\M}{\tensor{M}}
\newcommand{\UT}{\tensor{U}}
\newcommand{\Tb}{\tensor{T}_b}
\newcommand{\Tbinv}{\inv{\tensor{T}_b}}

\newcommand{\T}{\tensor{T}}

\newcommand{\HL}{\tensor{H_L}}
\newcommand{\HR}{\tensor{H_R}}

\newcommand{\HLRinv}{\inv{(\tensor{H_R}+\tensor{H_L})}}

\newcommand{\eyeT}{\tensor{I}}

\newcommand{\lammin}{\lambda_{\textrm{min}}}

\newcommand{\lamminH}{\lambda_{\textrm{min}}(\H)}
\newcommand{\init}{\Delta_0}
\newcommand{\lammaxH}{\lambda_{\textrm{max}}(\H)}

\newcommand{\infbound}{R^2}
\newcommand{\sqrtinfbound}{R}

\DeclareMathOperator{\Tr}{Tr}
\newcolumntype{C}[1]{>{\centering\let\newline\\\arraybackslash\hspace{0pt}}m{#1}}
\newcommand{\ind}{e}

\title{Parallelizing Stochastic Gradient Descent for Least Squares Regression: mini-batching, averaging, and model misspecification\footnote{This paper is published in the Journal of Machine Learning Research (JMLR), 2018.}}

\author[1]{Prateek Jain}
\author[2]{Sham M. Kakade}
\author[2]{Rahul Kidambi}
\author[1]{Praneeth Netrapalli}
\author[3]{Aaron Sidford}
\affil[1]{Microsoft Research, Bangalore, India,    \url{{prajain,praneeth}@microsoft.com}}
\affil[2]{University of Washington, Seattle, WA, USA,  \url{sham@cs.washington.edu},\ \url{rkidambi@uw.edu}}
\affil[3]{Stanford University, Palo Alto, CA, USA, \url{sidford@stanford.edu}.}

\date{}
\allowdisplaybreaks
\newtheorem{theorem}{Theorem}
\newtheorem{lemma}[theorem]{Lemma}
\newtheorem{corollary}[theorem]{Corollary}

\theoremstyle{definition}

\begin{document}


\maketitle

\begin{abstract}
This work characterizes the benefits of averaging techniques widely used in conjunction with stochastic gradient descent (SGD). In particular, this work presents a sharp analysis of: (1) mini-batching, a method of averaging many samples of \red{a stochastic} gradient to both reduce the variance of a stochastic gradient estimate and for parallelizing SGD and (2) tail-averaging, a method involving averaging the final few iterates of SGD in order to decrease the variance in SGD's final iterate. This work presents sharp finite sample generalization error bounds for these schemes for the stochastic approximation problem of least squares regression.

Furthermore, this work establishes a precise problem-dependent extent to which mini-batching can be used to yield provable near-linear parallelization speedups over SGD with batch size one. This characterization is used to understand the relationship between learning rate versus batch size when considering the excess risk of the final iterate of an SGD procedure. Next, this mini-batching characterization is utilized in providing a highly parallelizable SGD method that achieves the minimax risk with nearly the same number of serial updates as batch gradient descent, improving significantly over existing SGD-style methods.  Following this, a non-asymptotic excess risk bound for model averaging (which is a communication efficient parallelization scheme) is provided.

Finally, this work sheds light on fundamental differences in SGD's behavior when dealing with mis-specified models in the non-realizable least squares problem. This paper shows that maximal stepsizes ensuring minimax risk for the mis-specified case {\em must} depend on the noise properties.

The analysis tools used by this paper generalize the operator view of averaged SGD~\citep{defossez15} followed by developing a novel analysis in bounding these operators to characterize the generalization error. These  techniques are of broader interest in analyzing various computational aspects of stochastic approximation.

\end{abstract}

\section{Introduction and Problem Setup}
\label{sec:intro}
With the ever increasing size of modern day datasets, practical algorithms for machine learning are increasingly constrained to spend less time and use less memory. This makes it particularly desirable to employ simple streaming algorithms that generalize well in a few passes over the dataset. 

Stochastic gradient descent (SGD) is perhaps the simplest and most well studied algorithm that meets these constraints.  The algorithm repeatedly samples an instance from the stream of data and updates the current parameter estimate using the gradient of the sampled instance. Despite its simplicity, SGD has been immensely successful and is the de-facto method for large scale learning problems. The merits of SGD for large scale learning and the associated computation versus statistics tradeoffs is discussed in detail by the seminal work of~\citet{bottou07}.

While a powerful machine learning tool, SGD in its simplest forms is inherently serial. Over the past years, as dataset sizes have grown there have been remarkable developments in processing capabilities with multi-core/distributed/GPU computing infrastructure available in abundance. The presence of this computing power has triggered the development of parallel/distributed machine learning algorithms~(\cite{mann09,zinkevich11,bradley11,niu11,li14,zhang15b}) that possess the capability to utilize multiple cores/machines. However, despite this exciting line of work, it is yet unclear how to best parallelize SGD and fully utilize these computing infrastructures. 

This paper takes a step towards answering this question, by characterizing the behavior of constant stepsize SGD for the problem of strongly convex stochastic least square regression (LSR) under two averaging schemes widely believed to improve the performance of SGD. In particular, this work considers the natural parallelization technique of \emph{mini-batching}, where multiple data-points are processed simultaneously and the current iterate is updated by the average gradient over these samples, and combine it with variance reducing technique of \emph{tail-averaging}, where the average of many of the final iterates are returned as SGD's estimate of the solution.

In this work, parallelization arguments are structured through the lens of a {\em work-depth} tradeoff: {\em work} refers to the total computation required to reach a certain generalization error, and {\em depth} refers to the number of serial updates. Depth, defined in this manner, is \red{a reasonable estimate of} the runtime of the algorithm on a large multi-core architecture with shared memory, where there is no communication overhead, and has strong implications for parallelizability on other architectures.

\subsection{Problem Setup and Notations}\label{sec:setup}
We use boldface small letters ($\x, \w$ etc.) for vectors, boldface capital letters ($\A, \H$ etc.) for matrices and normal script font letters ($\M,\T$ etc) for tensors. We use $\otimes$ to denote the outer product of two vectors or matrices. Loewner ordering between two PSD matrices is represented using $\succeq,\preceq$.

\noindent This paper considers the stochastic approximation problem of Least Squares Regression (LSR). Let $L : \Rd \rightarrow \R$ be the expected square loss over tuples $(\x,y)$ sampled from a distribution $\D$:
\begin{align}
	\label{eq:objectiveFunction}
	L(\w) = \frac{1}{2}\cdot\mathbb{E}_{(\x,y)\sim\mathcal{D}} [\left(y-\langle \w, \x \rangle\right)^2] \; \forall \; \w \in \Rd.
\end{align}
Let $\w^*$ be a minimizer of the problem~\eqref{eq:objectiveFunction}. Now, let the Hessian of the problem~\eqref{eq:objectiveFunction} be denoted as:
\begin{align*}
\H \defeq \nabla^2 L(\w) = \E{\x\xTrans}.
\end{align*}
Next, we define the fourth moment tensor $\M$ of the inputs $\x$ as:
\begin{align*}
\M \defeq \E{\x\otimes\x\otimes\x\otimes\x}.
\end{align*}
Let the noise $\epsilon_{\x,y}$ in a sample $(\x,y)\sim\D$ with respect to the minimizer $\w^*$ of~\eqref{eq:objectiveFunction} be denoted as:
\begin{align*}
\epsilon_{\x,y} \defeq y - \iprod{\ws}{\x}.
\end{align*}
Finally, let the noise covariance matrix $\Sig$ be denoted as:
\begin{align*}
\Sig \defeq \E{\epsilon_{\x,y}^2\x\xTrans}.
\end{align*}
The {\em homoscedastic} (or, additive noise/well specified) case of LSR refers to the case when $\epsilon_{\x,y}$ is mutually independent from $\x$. This is the case, say, when $\epsilon_{\x,y}$ sampled from a Gaussian, $N(0,\sigma^2)$ independent of $\x$. In this case, $\Sig=\sigma^2\H$, where, $\sigma^2=\E{\epsilon^2}$, where the subscript on $\epsilon_{\x,y}$ is suppressed owing to the independence of $\epsilon$ on any sample $(\x,y)\sim\D$. On the other hand, the {\em heteroscedastic} (or, mis-specified) case refers to the setting when $\epsilon_{\x,y}$ is correlated with the input $\x$. In this paper, all our results apply to the general mis-specified case of the LSR problem.
\subsubsection{Assumptions}
We make the following assumptions about the problem.
\begin{itemize}
\item[$\left(\mathcal{A}1\right)$] \textbf{Finite fourth moment:} The fourth moment tensor $\M = \E{\x^{\otimes 4}}$ exists and is finite.
\item[$\left(\mathcal{A}2\right)$] \textbf{Strong convexity:} The Hessian of $L(\cdot)$, $\H = \E{\x\xTrans}$ is positive definite i.e., $\H \succ 0$.
\end{itemize}
$\left(\mathcal{A}1\right)$ is a standard regularity assumption for the analysis of SGD and related algorithms. $\left(\mathcal{A}2\right)$ is also a standard assumption and guarantees that the minimizer of~\eqref{eq:objectiveFunction}, i.e., $\ws$ is unique. 
\subsubsection{Important Quantities}
In this section, we will introduce some important quantities required to present our results. Let $\eye$ denote the $d\times d$ identity matrix. For any matrix $\A$, $\M \A \defeq \E{\left(\xTrans \A \x\right)\x\xTrans}$. Let $\HL=\H\otimes\eye$ and $\HR=\eye\otimes\H$ represent the left and right multiplication operators of the matrix $\H$ so that for any matrix $\A$, we have $\HL\A=\H\A$ and $\HR\A=\A\H$.
\begin{itemize}
	\item	\textbf{Fourth moment bound:} Let $R^2$ be the smallest number such that $\M\eye \preceq R^2 \H$.
	\item	\textbf{Smallest eigenvalue:} Let $\mu$ be the smallest eigenvalue of $\H$ i.e., $\H \succeq \mu \eye$.
\end{itemize}
The fourth moment bound implies that $\E{\|\x\|^2}\leq\infbound$. Further more, $\left(\mathcal{A}2\right)$ implies that the smallest eigenvalue $\mu$ of $\H$ is strictly greater than zero ($\mu>0$).
\subsubsection{Stochastic Gradient Descent: Mini-Batching and Iterate Averaging}
In this paper, we work with a stochastic first order oracle. This oracle, when queried at $\w$ samples an instance $(\x,y)\sim\D$ and uses this to return an unbiased estimate of the gradient of $L(\w)$:
\begin{align*}
\widehat{\nabla L}(\w) = -(y-\iprod{\w}{\x})\cdot\x;\ \ \E{\widehat{\nabla L}(\w)}=\nabla L(\w).
\end{align*}
We consider the stochastic gradient descent (SGD) method~\citep{robbins51}, which minimizes $L(\w)$ by following the direction opposite to this noisy stochastic gradient estimate, i.e.:
\begin{align*}
\w_{t} = \w_{t-1} - \gamma \cdot \widehat{\nabla L_t}(\w_{t-1}),\text{ with, } \widehat{\nabla L_t}(\w_{t-1}) = -(y_t-\iprod{\w_{t-1}}{\x_t})\cdot\x_t
\end{align*}
with $\gamma>0$ being a constant step size/learning rate; $\widehat{\nabla L_t}(\w_{t-1})$ is the stochastic gradient evaluated using the sample $(\x_t,y_t)\sim\D$ at $\w_{t-1}$. We consider two algorithmic primitives used in conjunction with SGD namely, mini-batching and tail-averaging (also referred to as iterate/suffix averaging).

Mini-batching involves querying the gradient oracle several times and using the average of the returned stochastic gradients to take a single step. That is,
\begin{align*}
\w_t = \w_{t-1} - \gamma \cdot \bigg(\frac{1}{b} \sum_{i=1}^b \widehat{\nabla L_{t,i}} (\w_{t-1})\bigg),
\end{align*}
where, $b$ is the batch size. Note that at iteration $t$, mini-batching involves repeatedly querying the stochastic gradient oracle at $\w_{t-1}$ for a total of $b$ times. For every query $i=1,...,b$ at iteration $t$, the oracle samples an instance $\{\x_{ti},y_{ti}\}$ and returns a stochastic gradient estimate $\widehat{\nabla L_{t,i}} (\w_{t-1})$. These estimates $\{\widehat{\nabla L_{t,i}} (\w_{t-1})\}_{i=1}^b$ are averaged and then used to perform a single step from $\w_{t-1}$ to $\w_{t}$. Mini-batching enables the possibility of parallelization owing to the use of cheap matrix-vector multiplication for computing stochastic gradient estimates. Furthermore, mini-batching allows for the possible reduction of variance owing to the effect of averaging several stochastic gradient estimates.

Tail-averaging (or suffix averaging) refers to returning the average of the final few iterates of a stochastic gradient method as a means to improve its variance properties~\citep{ruppert88,polyak92}. In particular, assuming the stochastic gradient method is run for $n-$steps, tail-averaging involves returning
\begin{align*}
\bar{\w} = \frac{1}{n-s}\sum_{t=s+1}^n \w_t
\end{align*}
as an estimate of $\ws$. Note that $s$ can be interpreted as being $cn$, with $c<1$ being some constant.

Typical excess risk bounds (or, generalization error bounds) for the stochastic approximation problem involve the contribution of two error terms namely, (i) the bias, which refers to the dependence on the starting conditions $\w_0$/initial excess risk $L(\w_0)-L(\ws)$ and, (ii) the variance, which refers to the dependence on the noise introduced by the use of a stochastic first order oracle.

\subsubsection{Optimal Error Rates for the Stochastic Approximation problem}
Under standard regularity conditions often employed in the statistics literature, the minimax optimal rate on the excess risk is achieved by the standard Empirical Risk Minimizer (or, Maximum Likelihood Estimator)~\citep{lehmann1998theory,Vaart00}. Given $n$ i.i.d. samples $\mathcal{S}_n=\{\x_i,y_i\}_{i=1}^n$ drawn from $\D$, define the empirical risk minimization problem as obtaining 
\begin{align*}
\w^*_n = \arg\min_{\w} \frac{1}{2n}\sum_{i=1}^n (y_i-\iprod{\w}{\x_i})^2.
\end{align*}
Let us define the noise variance $\widehat{\sigma^2_{\text{MLE}}}$ to represent
\begin{align*}
\widehat{\sigma^2_{\text{MLE}}} = \E{\|\widehat{\nabla L}(\w^*)\|^2_{\Hinv}} = \Tr[\Hinv\Sig].
\end{align*}
The asymptotic minimax rate of the Empirical Risk Minimizer $\w^*_n$ on {\em every problem instance} is $\widehat{\sigma^2_{\text{MLE}}}/n$~\citep{lehmann1998theory,Vaart00}, i.e.,
\begin{align*}
\lim_{n\to\infty} \frac{\mathbb{E}_{\mathcal{S}_n}[L(\w^*_n)] - L(\ws)} {\widehat{\sigma^2_{\text{MLE}}}/n}=1.
\end{align*}
For the well-specified case (i.e., the additive noise case, where, $\Sig=\sigma^2\H$), we have  $\widehat{\sigma^2_{\text{MLE}}}=d\sigma^2$. Seminal works of~\cite{ruppert88,polyak92} prove that tail-averaged SGD, with averaging from start, achieves the minimax rate for the {\em well-specified} case in the limit of $n\to\infty$. 

\indent\textit{Goal:} In this paper, we seek to provide a non-asymptotic understanding of (a) mini-batching and issues of learning rate versus batch-size, (b) tail-averaging, (c) the effect of the model mis-specification, (d) a batch size doubling scheme for parallelizing statistical estimation, (e) a communication efficient parallelization scheme namely, parameter-mixing/model averaging and (f) the behavior of learning rate versus batch size on the final iterate of the mini-batch SGD procedure, on the behavior of excess risk of SGD (in terms of both the bias and the variance terms) for the streaming LSR problem, with the goal of achieving the minimax rate on every problem instance.

\subsection{This Paper's Contributions} 
The main contributions of this paper are as follows:
\begin{itemize}
	\item	This work shows that mini-batching yields near-linear parallelization speedups over the standard serial SGD (i.e. with batch size $1$), as long as the mini-batch size is smaller than a problem dependent quantity (which we denote by $\bthresh$). When batch-sizes increase beyond $\bthresh$, mini-batching is inefficient (owing to the lack of serial updates), thus obtaining only sub-linear speedups over mini-batching with a batch size $\bthresh$. A by-product of this analysis sheds light on how the step sizes naturally interpolate from ones used by standard serial SGD (with batch size $1$) to ones used by batch gradient descent.
	\item	While the final iterate of SGD decays the bias at a geometric rate but does not obtain minimax rates on the variance, the averaged iterate~\citep{polyak92,defossez15} decays the bias at a sublinear rate while achieving minimax rates on the variance. This work rigorously shows that tail-averaging obtains the best of both worlds: decaying the bias at a geometric rate and obtaining near-minimax rates (up to constants) on the variance. This result corroborates with empirical findings~\citep{merityAveraging17} that indicate the benefits of tail-averaging in general contexts such as training Long-Short term memory models (LSTMs).
	\item Next, this paper precisely characterizes the tradeoffs of learning rate versus batch size and its effect on the excess risk of the final iterate of an SGD procedure, which provides theoretical evidence to empirical observations~\citep{goyal2017accurate,smith2017don} described in the context of deep learning and non-convex optimization.
	\item	Combining the above results, this paper provides a mini-batching and tail-averaging version of SGD that is highly parallelizable: the number of serial steps (which is a proxy for the un-parallelizable time) of this algorithm nearly matches that of \emph{offline gradient descent} and is lower than the serial time of all existing streaming LSR algorithms. See Table~\ref{tab:one} for comparison. We note that these results are obtained by providing a tight finite-sample analysis of the effects of mini-batching and tail-averaging with large constant learning rate schemes.
	\item We provide a non-asymptotic analysis of parameter mixing/model averaging schemes for the streaming LSR problem. Model averaging schemes are an attractive proposition for distributed learning owing to their communication efficient nature, and they are particularly effective in the regime when the estimation error (i.e. variance) is the dominating term in the excess risk. Here, we characterize the excess risk (in terms of both the bias and variance) of the model averaging procedure which sheds light on situations when it is an effective parallelization scheme (in that when this scheme yields linear parallelization speedups).
	\item All the results in this paper are established for the {\em general mis-specified} case of the streaming LSR problem. This establishes a fundamental difference in the behavior of SGD when dealing with mis-specified models in contrast to existing analyses that deal with the well-specified case. In particular, this analysis reveals a surprising insight that the maximal stepsizes (that ensure minimax optimal rates) are a function of the noise properties of the mis-specified problem instance. The main takeaway of this analysis is that the maximal step sizes (that permit achieving minimax rates) for the mis-specified case can be {\em much lower} than ones employed in the well-specified case: indeed, a problem instance that yields such a separation between the maximal learning rates for the well specified and the mis-specified case is presented.
\end{itemize}

The tool employed in obtaining these results generalizes the operator view of averaged SGD with batch size $1$~\citep{defossez15} and a clear exposition of the bias-variance decomposition from~\citet{jain2017markov} to obtain a sharp bound on the excess risk for mini-batch, tail-averaged constant step-size SGD. Note that the work of~\cite{defossez15} does not establish minimax rates while working with large constant step sizes; this shortcoming is remedied by this paper through a novel sharp analysis that rigorously establishes minimax optimal rates while working with large constant step sizes. Furthermore, note that while straightforward operator norm bounds of the matrix operators suffice to show convergence of the SGD method, they turn out to be pretty loose bounds (particularly for bounding the variance). To tighten these bounds, this paper presents a fine grained analysis that bounds the trace of the SGD operators when applied to the relevant matrices. The bounds of this paper and its advantages compared to existing algorithms is indicated in table~\ref{tab:one}.

While this paper's results focus on strongly convex streaming least square regression, we believe that our techniques and results extend more broadly. This paper aims to serve as the basis for future work on analyzing SGD and parallelization of large scale algorithms for machine learning.

\indent\textit{Paper organization}: Section~\ref{sec:relatedWork} presents the related work. Section~\ref{sec:mainResult} presents the main results of this work. Section~\ref{sec:poutline} outlines the proof techniques. Section~\ref{sec:experiments} presents experimental simulations to demonstrate the practical utility of the established mini-batching limits and tail-averaging. The proofs of all the claims and theorems are provided in the appendix. 

	\begin{table}[t!]
	\centering
		\begin{adjustbox}{max width=\textwidth}
		\begin{tabular}{ | c | c |c  | c | c | c |}
		\hline
		 {\bf Algorithm} & {\bf Final error} & {\bf Runtime/Work} & {\bf Depth} & {\bf Streaming} & {\bf Agnostic}\\ \hline
		 \begin{tabular}{c}Gradient Descent\\\citep{cauchy1847}\end{tabular} & $\order{\frac{\sigma^2 d}{n}}$ & $ \cnH n d \log \frac{n\cdot\init}{\sigma^2d} $ & $\cnH \log \frac{n\cdot\init}{\sigma^2d}$ & $\times$ & \checkmark \\ \hline
		 \begin{tabular}{c}SDCA\\\citep{shwartz12}\end{tabular} & $\order{\frac{\sigma^2 d}{n}}$ & $ (n + \frac{R^2}{\lammin} d) d\cdot\log \frac{n\cdot\init}{\sigma^2 d} $ & $(n + \frac{R^2}{\lammin} d)\cdot\log \frac{n\cdot\init}{\sigma^2 d}$ & $\times$ & $\checkmark$\\ \hline
		 \begin{tabular}{c}Averaged SGD\\\citep{defossez15}\tablefootnote{\citet{defossez15} guarantee these bounds with learning rate $\gamma\to 0$. This work supports these bounds with $\gamma=1/\infbound$.}\end{tabular} & $\mathcal{O}\left(\frac{1}{\lammin^2 n^2 \gamma^2} \cdot\init\right.$ $\left. + \frac{\sigma^2 d}{n}\right)$ & $nd$ & $n$ & \checkmark & $\times$\\ \hline
		{\begin{tabular}{c} Streaming SVRG\\with initial error oracle~\tablefootnote{Initial error oracle provides initial excess risk $\init=L(\w_0)-L(\ws)$ and noise level $\sigma^2$.}\\\citep{frostig15a}\end{tabular}} & $\mathcal{O}\left( \exp\left(-\frac{n\lamminH}{{\infbound}}\right)\cdot\init\right)+ \frac{\sigma^2d}{n} $ & $nd$ & \begin{tabular}{c}$(\frac{\infbound}{\lamminH})\cdot\log \frac{n\cdot\init}{\sigma^2d}$\end{tabular} & \checkmark & $\checkmark$\\ \hline
		\begin{tabular}{c}Algorithm~\ref{alg:DOminibatchSGD}\\(this paper)\end{tabular}
		 & \red{$\mathcal{O}\left( \left(\frac{\infbound t}{\twonorm{\H} n}\right)^\frac{t}{\cnH\log(\cnH)}\cdot\init+ \frac{ \sigma^2d}{n}\right) $} & $nd$ & \begin{tabular}{c}$ \frac{t}{t-\cnH \log(\cnH)}\cdot \cnH \log(\cnH)\cdot$\\ $\log\left(\frac{n\cdot \init}{\sigma^2 d} \cdot \frac{\infbound t}{\twonorm{\H}}\right)$\end{tabular} & \checkmark & \checkmark\\ \hline
		\begin{tabular}{c}Algorithm~\ref{alg:DOminibatchSGD}\\ with initial error oracle\\ (this paper)\end{tabular} & $\mathcal{O}\left( \exp\left(-\frac{n\lamminH}{{\infbound\cdot\log(\cnH)}}\right)\cdot\init+ \frac{\sigma^2d}{n} \right)$ & $nd$ & $\cnH \log(\cnH)\log \frac{n\cdot\init}{\sigma^2d}$ & \checkmark & \checkmark \\ \hline
	\end{tabular}
	\end{adjustbox}
\caption{Comparison of Algorithm~\ref{alg:DOminibatchSGD} with existing algorithms including offline methods such as Gradient Descent, SDCA and streaming methods such as averaged SGD, streaming SVRG given $n$ samples for LSR, with $\init=L(\w_0)-L(\ws)$. The error of offline methods are obtained by running these algorithms so that their final error is $\mathcal{O}(\sigma^2 d/n)$ (which is the minimax rate for the realizable case). The table is written assuming the realizable case; for algorithms which support agnostic case, these bounds can be appropriately modified. Refer to Section~\ref{sec:setup} for the definitions of all quantities. We do not consider accelerated variants in this table. Note that the accelerated variants have served to improve running times of the offline algorithms, with the sole exception of~\citet{JainKKNS17}. In the bounds for Algorithm~\ref{alg:DOminibatchSGD}, we require $t\geq 24\cnH\log(\cnH) $. Finally, note that streaming SVRG does not conform to the first order oracle model~(\cite{agarwal12}).}
		\label{tab:one}
	\end{table}

\vspace*{-3mm}

\section{Related Work}\label{sec:relatedWork}
Stochastic approximation has been the focus of much efforts starting with the work of~\cite{robbins51}, and has been analyzed in subsequent works including~\cite{nemirovsky83,kushner87,kushner03}. These questions and the related issues of computation versus statistics tradeoffs have received renewed attention owing to their relevance in the context of modern large scale machine learning, as highlighted by the work of~\cite{bottou07}.

\indent\textit{Geometric Rates on initial error:} For {\em offline optimization} with strongly convex objectives, gradient descent~\citep{cauchy1847} and fast gradient methods~\citep{Polyak64,Nesterov83} indicate linear convergence. However, a multiplicative coupling of number of samples $n$ and condition number in the computational effort is a major drawback in the large scale context. These limitations are addressed through developments in offline stochastic methods~\citep{roux12,shwartz12,johnson13,defazio14} and their accelerated variants~\citep{ShwartzZ13,frostig15b,LinMH15,Defazio16,Zhu16} which offer near linear running time in the number of samples and condition number with $\log(n)$ passes over the dataset stored in memory.

For {\em stochastic approximation} with strongly convex objectives, SGD offers linear rates on the bias without achieving minimax rates on the variance~\citep{bach11,needell16,bottou2016optimization}. In contrast, iterate averaged SGD~\citep{ruppert88,polyak92} offers a sub-linear $\mathcal{O}(1/n^2)$ rate on the bias~\citep{defossez15,dieuleveut15} while achieving minimax rates on the variance. Note that all these results consider the well-specified (additive noise) case when stating the generalization error bounds. We are unaware of any results that provide sharp non-asymptotic analysis of SGD and the related step size issues in the general mis-specified case. Streaming SVRG~\citep{frostig15a} offers a geometric rate on the bias and optimal statistical error rates; we will return to a discussion of Streaming SVRG below. In terms of methods faster than SGD, our own effort~\citep{JainKKNS17} provides the first accelerated stochastic approximation method that improves over SGD on every problem instance.

\indent\textit{Parallelization of Machine Learning algorithms:} In {\em offline optimization},~\citet{bradley11} study parallel co-ordinate descent for sparse optimization. Parallelization via mini-batching has been studied in~\citet{cotter11,takac13,shwartz13,takac15}. These results compare worst case upper bounds on the training error to argue parallelization speedups, thus providing weak upper bounds on mini-batching limits. Parameter mixing/Model averaging~\citep{mann09} guarantees linear parallelization speedups on the variance but do not improve the bias. Approaches that attempt to re-conciliate communication-computation tradeoffs~\citep{li14} indicate increased mini-batching hurts convergence, and this is likely an artifact of comparing weak upper bounds. Hogwild~\citep{niu11} indicates near-linear parallelization speedups in the harder asynchronous optimization setting, relying on specific input structures like hard sparsity; these bounds are obtained by comparing worst case upper bounds on training error. Refer to oracle models paragraph below for details on these worst case upper bounds.

In the {\em stochastic approximation} context,~\citet{dekel12} study mini-batching in an oracle model that assumes bounded variance of stochastic gradients. These results compare worst case bounds on the generalization error to prescribe mini-batching limits, which renders these limits to be too loose (as mentioned in their paper). Our paper's mini-batching result offers guidelines on batch sizes for linear parallelization speedups by comparing generalization bounds that hold on a per problem basis as opposed to worst case bounds. Refer to the paragraph on oracle models for more details. Finally, parameter mixing in the stochastic approximation context~\citep{rosenblatt14,zhang15a} offers linear parallelization speedups on the variance error while not improving the bias~\citep{rosenblatt14}. Finally,~\citet{duchi15} guarantees asymptotic optimality of asynchronous optimization with linear parallelization speedups on the variance.

\indent\textit{Oracle models and optimality:} In stochastic approximation, there are at least two lines of thought with regards to oracle models and notions of optimality. One line involves considering the case of bounded noise~\citep{kushner03,KushnerClark}, or, bounded variance of the stochastic gradient, which in the least squares setting amounts to assuming bounds on 
\begin{align*}
\widehat{\nabla L}(\w)-\nabla L(\w) = (\x\x^\top-\H)(\w-\ws) - \epsilon\x.
\end{align*}
This implies additional assumptions are required on compactness of the parameter set (which are enforced via projection steps); such assumptions do not hold in practical implementation of stochastic gradient methods and in the setting considered by this paper. Thus, the mini-batching thresholds in ~\citet{cotter11,niu11,dekel12,li14} present bounds in the above worst-case oracle model by comparing weak upper bounds on the training/test error. 

Another view of optimality~\citep{anbar1971optimal,Fabian:1973:AES} considers an objective where the goal is to match the rate of the statistically optimal estimator (referred to as the $M-$estimator) on every problem instance. \citet{polyak92} consider this oracle model for the LSR problem and prove that the distribution of the averaged SGD estimator on every problem matches that of the $M-$estimator under certain regularity conditions~\citep{lehmann1998theory}. A recent line of work~\citep{bach13,frostig15a} aims to provide non-asymptotic guarantees for SGD and its variants in this oracle model. This paper aims to understand mini-batching and other computational aspects of parallelizing stochastic approximation on every problem instance by working in this practically relevant oracle model. Refer to~\citet{JainKKNS17} for more details.

\indent\textit{Comparing offline and streaming algorithms:} Firstly, offline algorithms require performing multiple passes over a dataset stored in memory. Note that results and convergence rates established in the finite sum/offline optimization context do not translate to rates on the generalization error. Indeed, these results require going though concentration and a generalization error analysis for this translation to occur. Refer to~\citet{frostig15a} for more details.

\indent\textit{Comparison to streaming SVRG:} Streaming SVRG does not function in the stochastic first order oracle model~\citep{agarwal12} satisfied by SGD as run in practice since it requires gradients at two points from a single sample~\citep{frostig15a}. Furthermore, in contrast to this work, its depth bounds depend on a stronger fourth moment property due to lack of mini-batching.

\section{Main Results}
\label{sec:mainResult}
We begin by writing out the behavior of the learning rate as a function of batch size.

\indent\textit{Maximal Learning Rates:} We write out a characterization of the largest learning rate $\gammabmaxdiv$ that permits the convergence of the mini-batch Stochastic Gradient Descent update. The following generalized eigenvector problem allows for the computation of $\gammabmaxdiv$:
\begin{align}
\label{eq:divLearnRate}
\frac{2}{\gammabmaxdiv}=\sup_{\W\in\SRd}\frac{\iprod{\W}{\M\W}+(b-1)\cdot\Tr{\W\H\W\H}}{b\cdot\Tr{\W\H\W}}.
\end{align}
This characterization generalizes the divergent stepsize characterization of~\cite{defossez15} for batch sizes $>1$. The derivation of the above characterization can be found in appendix~\ref{ssec:divLearnRateDerivation}. We note that this characterization sheds light on how the divergent learning rates interpolate from batch size $1$ (which is $\leq2/\Tr{\H})$ to the batch gradient descent learning rate (setting $b$ to $\infty$), which turns out to be $2/\lammaxH$. A property of $\gammabmaxdiv$ worth noting is that it does not depend on properties of the noise ($\Sig$), and depends only on the second and fourth moment properties of the covariate $\x$. 

We note that in this paper, our interest does not lie in the non-divergent stepsizes $0\leq\gamma\leq\gammabmaxdiv$, but in the set of (maximal) stepsizes $0\leq\gamma\leq\gammabmax$ ($<\gammabmaxdiv$) that are sufficient to guarantee minimax error rates of $\mathcal{O}(\widehat{\sigma^2_{\text{MLE}}}/n)$. For the LSR problem, these maximal learning rates $\gammabmax$ are:
\begin{align}
\label{eq:gammabmaxStatOpt}
\gammabmax \defeq \frac{2b}{\infbound\cdot\rhom +(b-1)\|\H\|_2}, \text{ where, } \rhom\defeq\frac{d\|(\HL+\HR)^{-1}\Sig\|_2}{\trace{(\HL+\HR)^{-1}\Sig}}.
\end{align}
Note that $\rhom\geq1$ captures a notion of ``degree'' of model mismatch, and how it impacts the learning rate $\gammabmax$; for the additive noise/well specified/homoscedastic case, $\rhom=1$. Thus, for problems where $\infbound$ and $\twonorm{\H}$ is held the same, the well-specified variant of the LSR problem admits a strictly larger learning rate (that achieves minimax rates on the variance) compared to the mis-specified case. Furthermore, in stark contrast to the well-specified case, $\gammabmax$ in the mis-specified case depends not just on the second and fourth moment properties of the input, but also on the noise covariance $\Sig$. We show that our characterization of $\gammabmax$ in the mis-specified case is tight in that there exist problem instances where $\gammabmax$ (equation~\ref{eq:gammabmaxStatOpt}) is off the maximal learning rate in the well-specified case (obtained by setting $\rhom=1$ in equation~\ref{eq:gammabmaxStatOpt}) by a factor of the dimension $d$ and $\gammabmax$ is still the largest step size yielding minimax rates. We also note that there could exist mis-specified problem instances where a step size $\gamma$ exceeding $\gammabmax$ achieves minimax rates. Characterizing the maximal learning rate that achieves minimax rates on \emph{every mis-specified} problem instance is an interesting open question. We return to the characterization of $\gammabmax$ in section~\ref{ssection:tradeoffs}.

Note that this paper characterizes the performance of Algorithms~\ref{alg:mbSGD} and~\ref{alg:DOminibatchSGD} when run with a step size $\gamma \leq \frac{\gamma_{b,\max}}{2}$. The proofs turn out to be tedious for $\gamma \in \left(\frac{\gammabmax}{2},\gammabmax \right)$ and can be found in the initial version of this paper~\citet{jain2016parallelizing} and these were obtained through generalizing the operator view of analyzing SGD methods introduced by~\citet{defossez15}. For the well-specified case, this paper's results hold for the same learning rate regimes as \citet{bach13,frostig15a}, that are known to admit statistical optimality. We also note that in the additive noise case, we are unaware of a separation between $\gammabmax$ and $\gammabmaxdiv$; but as we will see, this is not of much consequence given that there exists a strict separation in the learning rate $\gammabmax$ between the well-specified and mis-specified problem instances.

Finally, note that the stochastic process viewpoint allows us to work with learning rates that are significantly larger compared to standard analyses that use function value contraction e.g.,~\citet[Theorem 4.6]{bottou2016optimization}. All existing works establishing mini-batching thresholds in the stochastic optimization setting e.g., \citet{dekel12} work in the worst case (bounded noise) oracle with small step sizes, and draw conclusions on mini-batch thresholds and effects by comparing weak upper bounds on the excess risk. 

\begin{algorithm}[t]
	\caption{\minibatchSGD} 
	\label{alg:mbSGD} 
	\begin{algorithmic}[1]
		\renewcommand{\algorithmicrequire}{\textbf{Input: }}
		\renewcommand{\algorithmicensure}{\textbf{Output: }}
		\REQUIRE Initial point $\w_0$, stepsize $\gamma$, minibatch size $b$, initial iterations $s$, total samples $n$.
		\FOR{$t=1,2,..,\lfloor{\frac{n}{b}}\rfloor$}
		\STATE Sample ``$b$'' tuples $\{(x_{ti},y_{ti})\}_{i=1}^b\sim\mathcal{D}^b$
		\STATE $\w_{t}\leftarrow \w_{t-1} - \gb \sum_{i=1}^b \widehat{\nabla L_{ti}}(\w_{t-1})$
		\ENDFOR
		\ENSURE $\bar{\w}=\tfrac{1}{\lfloor{\frac{n}{b}}\rfloor-s}\sum_{i>s}\w_i$
	\end{algorithmic}
\end{algorithm}

\indent\textit{Mini-Batched Tail-Averaged SGD for the mis-specified case:} We present our main result, which is the error bound for mini-batch tail-averaged SGD for the general mis-specified LSR problem.
\begin{theorem}\label{thm:mainMBTANR}
	Consider the general mis-specified case of the LSR problem~\ref{eq:objectiveFunction}. Running Algorithm~\ref{alg:mbSGD} with a batch size $b\geq 1$, step size $\gamma\leq\gammabmax/2$, number of unaveraged iterations $s$, total number of samples $n$, we obtain an iterate $\wbar$ satisfying the following excess risk bound:
	\begin{align}
	\label{eq:mbtagenError}
	\E{L(\wbar)} - L(\ws) \leq \frac{2}{\gamma^2\mu^2}\cdot\frac{(1-\gamma\mu)^s}{\big(\frac{n}{b}-s\big)^2}\cdot\big(L(\w_0)-L(\ws)\big) + 4\cdot\frac{\widehat{\sigma^2_{\text{MLE}}}}{b\cdot(\frac{n}{b}-s)}.
	\end{align}
	In particular, with $\gamma=\gammabmax/2$, we have the following excess risk bound:
	\begin{align*}
	&L(\wbar) - L(\ws) \leq \underbrace{\frac{2 \cnH_b^2}{\left(\frac{n}{b}-s\right)^2}\exp\left(-\frac{s}{\cnH_b}\right)\big(L(\w_0)- L(\ws)\big)}_{\Tone}  + \underbrace{4\cdot\frac{\widehat{\sigma^2_{\text{MLE}}}}{b(\frac{n}{b}-s)}}_{\Ttwo},
	\end{align*}
	with $\cnH_b = \frac{\infbound\cdot\rhom+(b-1)\twonorm{\H}}{b\lamminH}$.
\end{theorem}
Note that the above theorem indicates that the excess risk is composed of two terms, namely the bias ($\Tone$), which represents the dependence on the initial conditions $\w_0$ and the variance ($\Ttwo$), which depends on the statistical noise ($\widehat{\sigma^2_{\text{MLE}}}$); the bias decays geometrically during the ``$s$'' unaveraged iterations while the variance is minimax optimal (up to constants) provided $s=\mathcal{O}(n)$. We will understand this geometric decay on the bias more precisely.

\indent\textit{Effect of tail-averaging SGD's iterates:} To understand tail-averaging, we specialize theorem~\ref{thm:mainMBTANR} with a batch size $1$ to the well-specified case, i.e., where, $\Sig=\sigma^2\H$, $\widehat{\sigma^2_{\text{MLE}}}=d\sigma^2$ and $\rhom=1$.\vspace*{-2mm}
\begin{corollary}\label{thm:TALemma}
	Consider the well-specified (additive noise) case of the streaming LSR problem ($\Sig=\sigma^2\H$), with a batch size $b=1$. With a learning rate $\gamma=\frac{\gammabmax[1]}{2}=\frac{1}{\infbound}$, unaveraged iterations $s$ and total samples $n$, we have the following excess risk bound:
	\begin{align*}
	&L(\wbar) - L(\ws) \leq \underbrace{\frac{2 \cnH_1^2}{\left(n-s\right)^2}\exp\left(-\frac{s}{\cnH_1}\right)\{L(\w_0)- L(\ws)\}}_{\Tone}  + \underbrace{4\cdot\frac{d\sigma^2}{n-s}}_{\Ttwo}, \text{where, } \cnH_1 = \infbound/\mu.
	\end{align*}
\end{corollary}
Tail-averaging allows for a geometric decay of the initial error $\Tone$, while tail-averaging over $s=c\cdot n$ (with $c<1$), allows for the variance $\Ttwo$ to be minimax optimal (up to constants). We note that the work of~\citet{merityAveraging17}, which studies empirical optimization for training non-convex sequence models (e.g. Long-Short term memory models (LSTMs)) also indicate the benefits of tail-averaging.

Note that this particular case (i.e. additive noise/well-specified case with batch size $1$) with tail-averaging from start ($s=0$) is precisely the setting considered in~\citet{defossez15}, and their result (a) achieves a sub-linear $\mathcal{O}(1/n^2)$ rate on the bias and (b) their variance term is shown to be minimax optimal only with learning rates that approach zero (i.e. $\gamma\to 0$). 
\subsection{Effects Of Learning Rate, Batch Size and The Role of Mis-specified Models}
\label{ssection:tradeoffs}
We now consider the interplay of learning rate, batch size and how model mis-specification plays into the mix. Towards this, we split this section into three parts: (a) understanding learning rate versus mini-batch size in the well-specified case, (b) how model mis-specification leads to a significant difference in the behavior of SGD and (c) how model mis-specification manifests itself when considered in tradeoff between the learning rate versus batch-size.

\indent\textit{Effects of mini-batching in the well-specified case:} As mentioned previously, in the well-specified case, $\Sig=\sigma^2\H$ and $\rhom=1$. For this case, equation~\eqref{eq:gammabmaxStatOpt} can be specialized as:
\begin{align}
\label{eqn:realizableCaseLearningRate}
\gammabmax = \frac{2b}{\infbound + (b-1)\|\H\|_2}.
\end{align}
Observe that the learning rate $\gammabmax$ grows linearly as a function of the batch size $b$ until a batch size $b=\bthresh = 1 + \frac{\infbound}{\twonorm{\H}}$. In the regime of batch sizes $1<b\leq\bthresh$, the resulting mini-batch SGD updates offer near-linear parallelization speedups over SGD with a batch size of $1$. Furthermore, increasing batch sizes beyond $\bthresh$ leads to sub-linear increase in the learning rate, and this implies that we lose the linear parallelization speedup offered by mini-batching with a batch-size $b\leq\bthresh$. Losing the linear parallelization is indicative of the following: consider the case when we double batch-size from $b>\bthresh$ to $2b$. Suppose the bias error $\Tone$ is larger than the variance $\Ttwo$, we require performing the same number of updates with a batch size $2b$ as we did with a batch size $b$ to achieve a similar excess risk bound; this implies we are inefficient in terms of number of samples (or, number of gradient computations) used to achieve a given excess risk. When the estimation error ($\Ttwo$) dominates the approximation error ($\Tone$), we note that larger batch sizes $b$ (with $b>\bthresh$) serves to improve the variance term, thus allowing linear parallelization speedups via mini-batching.

Note that with a batch size of $b=\bthresh$, the learning rate of $\mathcal{O}(1/\lammaxH)$ employed by mini-batch SGD resembles ones used by batch gradient descent. This mini-batching characterization thus allows for understanding tradeoffs of learning rate versus batch size. This behavior is noted in practice (empirically, but with no underlying rigorous theory) for a variety of problems (going beyond linear regression/convex optimization), in the deep learning context~\citep{goyal2017accurate}.

\indent\textit{SGD's behaviour with mis-specified models:} Next, this paper attempts to shed light on some fundamental differences in the behavior of SGD when dealing with the mis-specified case (as against the well-specified case, which is the focus of existing results~\citep{polyak92,bach13,dieuleveut15,defossez15}) of the LSR problem. This paper's results in general mis-specified case with batch sizes $b>1$ specialize to existing results additive noise/well-specified case with batch size $1$~\citep{bach13,dieuleveut15}. To understand these issues better, we consider $\gammabmax$ in equation~\ref{eq:gammabmaxStatOpt} with a batch size $1$:
\begin{align}
\label{eq:learnRateBS1}
\gammabmax[1] = \frac{2}{\infbound\cdot\rhom}.
\end{align}
Recounting that $\rhom\geq 1$, observe that the mis-specified case admits a maximal learning rate (with a view of achieving minimax rates) that is at most as large as the additive noise/well-specified case, where $\rhom=1$. Note that when $\trace{\HL+\HR)^{-1}\Sig}$ is nearly the same (say, upto constants) as the spectral norm $\twonorm{\HL+\HR)^{-1}\Sig}$, then $\rhom=\mathcal{O}(d)$ and $\gammabmax[1]=\mathcal{O}(\tfrac{1}{\infbound d})$. This implies that there exist mis-specified models whose noise properties (captured through the noise covariance matrix $\Sig$) prevents SGD from working with large learning rates of $\mathcal{O}(1/\infbound)$ used in the well-specified case.

This notion is formalized in the following lemma, which presents an instance working with the mis-specified case, wherein, SGD {\em cannot} employ large learning rates used by the well-specified variant of the problem, while {\em retaining minimax optimality}. This behavior is in stark contrast to algorithms such as streaming SVRG~(\cite{frostig15a}), which work with the same large learning rates in the mis-specified case as in the well-specified case, while guaranteeing minimax optimal rates. The proof of lemma~\ref{lem:agnosticLowerBound1} can be found in the appendix~\ref{ssection:sep}.

\begin{lemma}
	\label{lem:agnosticLowerBound1}
	Consider a Streaming LSR example with Gaussian covariates (i.e. $\x\sim\mathcal{N}(0,\H)$) with a diagonal second moment matrix $\H$ that is defined by:
	\[ \H_{ii} =
	\begin{cases}
	1       & \quad \mathrm{if} \ i=1\\
	1/d  & \quad \mathrm{if} \ i>1\\
	\end{cases}.
	\]
	Further, let the noise covariance matrix $\Sig$ be diagonal as well, with the following entries:
	\[ \Sig_{ii} =
	\begin{cases}
	1       & \quad \mathrm{if} \ i=1\\
	1/[(d-1)d]  & \quad \mathrm{if} \ i>1\\
	\end{cases}.
	\]
	For this problem instance, $\gammabmax[1]\leq\frac{4}{(d+2)(1+\frac{1}{d})}$ is necessary for retaining minimax rates, while the well-specified variant of this problem permits a maximal learning rate $\leq\frac{d}{(d+2)(1+\frac{1}{d})}$, thus implying an $\mathcal{O}(d)$ separation in learning rates between the well-specified and mis-specified case.
\end{lemma}

\indent\textit{Learning rate versus mini-batch size issues in the mis-specified case:} Noting that for the batch size $1$, as mentioned in equation~\ref{eq:learnRateBS1}, the learning rate for the mis-specified case in the most optimistic situation (when $\rhom=\text{constant}$) can be atmost as large as the learning rate for the well-specified case. Furthermore, we also know from the observations in the mis-specified case that the learning rate tends to grow linearly as a function of the batch size until it hits the limit of $\mathcal{O}(1/\lammaxH)$. Combining these observations, we will revisit equation~\ref{eq:gammabmaxStatOpt}, which says:
\begin{align*}
\gammabmax \defeq \frac{2b}{\infbound\cdot\rhom +(b-1)\|\H\|_2}.
\end{align*}
This implies that the mini-batching size threshold $\bthresh$ can be expressed as:
\begin{align}
\label{eq:bthreshAgnostic}
\bthresh \defeq  1 + \frac{\infbound}{\twonorm{\H}}\cdot\rhom.
\end{align}
When $1<b\leq\bthresh$, we achieve near linear parallelization speedups over running SGD with a batch size $1$. Note that this characterization specializes to the batch size threshold $\bthresh$ presented in the well-specified case (i.e. where $\rhom=1$). Furthermore, this batch size threshold (in the mis-specified case) could be much larger than the threshold in the well-specified case, which is expected since the learning rate for a batch size $1$ in the mis-specified case can potentially be much smaller than ones used in the well specified case. Furthermore, with a batch size $\bthresh$, note that the learning rate is $\mathcal{O}(1/\lammaxH)$, resembling ones used with batch gradient descent.

\indent\textit{Behavior of the final-iterate:} We now present the excess risk bound offered by the final iterate of a stochastic gradient scheme. This result is of much practical relevance in the context of modern machine learning and deep learning, where final iterate is often used, and where the tradeoffs between learning rate and batch sizes are discussed in great detail~\citep{smith2017don}. For this discussion, we consider the well-specified case to present our results owing to its ease in presentation. Our framework and results are generic for translating these observations to the mis-specified case.
\begin{lemma}\label{lem:lastPointLemma}
	Consider the well-specified case of the LSR problem. Running Algorithm~\ref{alg:mbSGD} with a step size $\gamma\leq\frac{\gammabmax}{2}=\frac{b}{\infbound+(b-1)\|\H\|_2}$, batch size $b$, total samples $n$ and with no iterate averaging (i.e. with $s=n-1$) yields a result $\w_{\lfloor n/b\rfloor}$ that satisfies the following excess risk bound:
	\begin{align}\label{eq:lp1}
	\E{L(\w_{\lfloor n/b\rfloor})}-L(\ws)\leq \cnH_b (1-\gamma\mu)^{\lfloor n/b \rfloor}\bigg(L(\w_0)-L(\ws)\bigg) + \gb \sigma^2\Tr{(\H)},
	\end{align}
	where $\cnH_b \defeq \frac{\infbound + (b-1) \|\H\|_2}{b\mu}$.
	In particular, with a step size $\gamma = \frac{\gammabmax}{2}=\frac{b}{\infbound+(b-1)\|\H\|_2}$, we have:
	\begin{align}\label{eq:lp2}
	\E{L(\w_{\lfloor n/b\rfloor})}-L(\ws)\leq \cnH_b\cdot e^{-\frac{\lfloor n/b\rfloor}{\cnH_b}}\cdot\bigg(L(\w_0)-L(\ws)\bigg) + \frac{\sigma^2\Tr{(\H)}}{\infbound + (b-1) \|\H\|_2}.
	\end{align}
\end{lemma}
\indent\textit{Remarks:} Noting that $\Tr{(\H)}\leq\infbound$, the variance of the final iterate with batch size $1$ is $\leq\sigma^2$. Next, with a batch size $b=\bthresh$, the final iterate has a variance $\leq\sigma^2/2$; at cursory glance this may appear interesting, in that by mini-batching, we do not appear to gain much in terms of the variance. This is unsurprising given that in the regime of $b\leq\bthresh$, the $\gammabmax$ grows linearly, thus nullifying the effect of averaging multiple stochastic gradients. Furthermore, this follows in accordance with the linear parallelization speedups offered by a batch size $1<b\leq\bthresh$. Note however, once $b>\bthresh$, any subsequent increase in batch sizes allows the variance of the final iterate to behave as $\mathcal{O}(\sigma^2/b)$. Finally, note that once $b>\bthresh$, doubling batch sizes $b$ (in equation~\ref{eq:lp2}) possesses the same effect as halving learning rate from $\gamma$ to $\gamma/2$ (as seen from equation~\ref{eq:lp1}), providing theoretical rigor to issues explored in training practical deep models~\citep{smith2017don}.
\subsection{Parallelization via Doubling Batch Sizes and Model Averaging}
\label{ssection:parallelizationDoublingModelAvging}
We now elaborate on a highly parallelizable stochastic gradient method, which is epoch based and relies on doubling batch sizes across epochs to yield an algorithm that offers the same generalization error as that of offline (batch) gradient descent in nearly the {\emph same} number of serial updates as batch gradient descent, while being a streaming algorithm that does not require storing the entire dataset in memory. Following this, we present a non-asymptotic bound for parameter mixing/model averaging, which is a communication efficient parallelization scheme that has favorable properties when the estimation error (i.e. variance) is the dominating term of the excess risk.

\indent\textit{(Nearly) Matching the depth of Batch Gradient Descent:} The result of theorem~\ref{thm:mainMBTANR} establishes a scalar generalization error bound of Algorithm~\ref{alg:mbSGD} for the general mis-specified case of LSR and showed that the depth (number of sequential updates in our algorithm) is decreased to $n/b$. This section builds upon this result to present a simple and intuitive doubling based streaming algorithm that works in epochs and processes a total of $n/2$ points. In each epoch, the minibatch size is increased by a factor of $2$ while applying Algorithm~\ref{alg:mbSGD} (with no tail-averaging) with twice as many samples as the previous epoch. After running over $n/2$ samples using this epoch based approach, we run Algorithm~\ref{alg:mbSGD} (with tail-averaging) with the remaining $n/2$ points. Note that each epoch decays the bias of the previous epoch linearly and halves the statistical error (since we double mini-batch size). The final tail-averaging phase ensures that the variance is small.

The next theorem formalizes this intuition and shows Algorithm~\ref{alg:DOminibatchSGD} improves the depth exponentially from $n/\bthresh$ to $\mathcal{O}\left(\kappa \log (d\kappa)\log(n\{L(\w_0)-L(\ws)\}/\widehat{\sigma^2_{\text{MLE}}})\right)$ in the presence of an error oracle that provides us with the initial excess risk $L(\w_0)-L(\ws)$ and the noise level $\widehat{\sigma^2_{\text{MLE}}}$.
\begin{algorithm}[t]
	\caption{\DOminibatchSGD}
	\label{alg:DOminibatchSGD}
	\begin{algorithmic}[1]
		\renewcommand{\algorithmicrequire}{\textbf{Input: }}
		\renewcommand{\algorithmicensure}{\textbf{Output: }}
		\REQUIRE Initial point $\w_0$, stepsize $\gamma$, initial minibatch size $b$, number of iterations in each epoch $s$, number of samples $n$.
		\STATE /*Run logarithmic number of epochs where each epoch runs $t$ iterations of minibatch SGD (with out averaging). Double minibatch size after each epoch.*/
		\FOR{$\ell = 1, 2, \cdots, \log \frac{n}{bt}-1$}
		\STATE $\bl \leftarrow 2^{\ell -1} b$
		\STATE $\w_{\ell} \leftarrow $ \minibatchSGD$(\w_{\ell-1},\gamma,b_{\ell},t-1,t\cdot b_{\ell})$
		\ENDFOR
		\STATE /*For the last epoch, run tail averaged minibatch SGD with initial point $\w_t$, stepsize $\gamma$, minibatch size $2^{\log \frac{n}{bt}-1} \cdot b = n/2t$, number of initial iterations $t/2$ and number of samples $n/2$.*/
		\STATE $\wbar \leftarrow $ \minibatchSGD$(\w_s,\gamma,n/2t,t/2,n/2)$
		\ENSURE $\wbar$
	\end{algorithmic}
\end{algorithm}
\begin{theorem}\label{thm:DOminibatchSGD}
	Consider the general mis-specified case of LSR. Suppose in Algorithm~\ref{alg:DOminibatchSGD}, we use initial batchsize of $b = \bthresh$, stepsize $\gamma = \frac{\gammabmax}{2}$ and number of iterations in each epoch being $t \geq 24\cnH \log(\cnH)$, we obtain the following excess risk bound on $\wbar$:
	\begin{align*}
	\E{L(\wbar)} - L(\ws)  \leq \left(\frac{2bt}{n}\right)^{\frac{t}{12\cnH\log(\cnH)}} \cdot\big(L(\w_0)- L(\ws)\big) + 80\frac{\widehat{\sigma^2_{\text{MLE}}}}{n}.
	\end{align*}
\end{theorem}

\textbf{Remarks}: The final error again has two parts: the bias term that depends on the initial error $L(\w_0)-L(\ws)$ and the variance term that depends on the statistical noise $\widehat{\sigma^2_{\text{MLE}}}$. Note that the variance error decays at a rate of $\order{\widehat{\sigma^2_{\text{MLE}}}/n}$ which is minimax optimal up to constant factors.

Algorithm~\ref{alg:DOminibatchSGD} decays the bias at a superpolynomial rate by choosing $t$ large enough. If Algorithm~\ref{alg:DOminibatchSGD} has access to an initial error oracle that provides $L(\w_0)-L(\ws)$ and $\widehat{\sigma^2_{\text{MLE}}}$, we can run Algorithm~\ref{alg:DOminibatchSGD} with a batch size $\bthresh$ until the excess risk drops to the noise level $\widehat{\sigma^2_{\text{MLE}}}$ and subsequently begin doubling the batch size. Such an algorithm indeed gives geometric convergence with a generalization error bound as:
{\small	\begin{align*}
	\E{L(\wbar)} - L(\ws)  &\leq \exp\left(-(\frac{n\lammin}{\infbound\cdot\log(\cnH)})\cdot\frac{1}{\rhom}\right) \{L(\w_0)-L(\ws)\} + 80\frac{\widehat{\sigma^2_{\text{MLE}}}}{n},
	\end{align*}}
with a depth of $\order{\cnH \log(d\cnH) \log\frac{n\{L(\w_0)-L(\ws)\}}{\widehat{\sigma^2_{\text{MLE}}}}}$. The proof of this claim follows relatively straightforwardly from the proof of Theorem~\ref{thm:DOminibatchSGD}. We note that this depth nearly matches (up to $\log$ factors), the depth of standard offline gradient descent despite being a streaming algorithm. This algorithm (aside from tail-averaging in the final epoch) resembles empirically effective schemes proposed in the context of training deep models~\citep{smith2017don}.

\indent\textit{Parameter Mixing/Model-Averaging:} We consider a communication efficient method for distributed optimization which involves running mini-batch tail-averaged SGD independently on $P$ separate machines (each containing their own independent samples) and averaging the resulting solution estimates. This is a well studied scheme for distributed optimization~\citep{mann09,zinkevich11,rosenblatt14,zhang15a}. As mentioned in~\citet{rosenblatt14}, these schemes do not appear to offer improvements in the bias error while offering near linear parallelization speedups on the variance. We provide here a non-asymptotic characterization of the behavior of model averaging for the general mis-specified LSR problem.

\begin{theorem}
	\label{thm:parameterMixing}
	Consider running Algorithm~\eqref{alg:mbSGD}, i.e., mini-batch tail-averaged SGD (for the mis-specified LSR problem~\eqref{eq:objectiveFunction}) independently in $P$ machines, each of which contains $N/P$ samples. Let algorithm~\eqref{alg:mbSGD} be run with a batch size $b$, learning rate $\gamma\leq\gammabmax/2$, tail-averaging begun after $s-$iterations, and let each of these machines output $\{\wbar_i\}_{i=1}^P$. The excess risk of the model-averaged estimator $\wbar=\tfrac{1}{P}\sum_{i=1}^P \wbar_i$ is upper bounded as:
	\begin{align*}
	\E{L(\wbar)}-L(\ws) &\leq \frac{(1-\gamma\mu)^{s}}{\gamma^2\mu^2\big(\frac{n}{P\cdot b}-s\big)^2}\cdot\frac{2+(P-1)(1-\gamma\mu)^{s}}{P}\cdot\bigg(L(\w_0)-L(\ws)\bigg)\\\nonumber&\qquad\qquad\qquad\qquad\qquad\qquad\qquad+4\cdot\frac{\widehat{\sigma^2_{\text{MLE}}}}{P\cdot b\cdot\big(\frac{n}{P\cdot b}-s\big)}.
	\end{align*}
	In particular, with $\gamma=\gammabmax/2$, we have the following excess risk bound:
	\begin{align*}
	\E{L(\wbar)}-L(\ws)&\leq\exp\bigg(-\frac{s}{\cnH_b}\bigg)\cdot\frac{\cnH_b^2}{\big(\frac{n}{P\cdot b}-s\big)^2}\cdot\frac{2+(P-1)\cdot\exp(-s/\cnH_b)}{P}\cdot\big(L(\w_0)-L(\ws)\big)\nonumber\\&\qquad\qquad\qquad\qquad\qquad\qquad\qquad\qquad+4\cdot\frac{\widehat{\sigma^2_{\text{MLE}}}}{P\cdot b\cdot\big(\frac{n}{P \cdot b}-s\big)}.
	\end{align*}
\end{theorem}
\indent\textit{Remarks:} We note that during the iterate-averaged phase (i.e. $t>s$), there is no reduction of the bias, whereas, during the (initial) unaveraged iterations, once $s>\cnH_b\log(P)$, we achieve linear speedups on the bias. We note that model averaging offers linear parallelization speedups on the variance error. Furthermore, when the bias reduces to the noise level, model averaging offers linear parallelization speedups on the overall excess risk. Note that if $s = c\cdot n/(P\cdot b)$, with $c<1$, then the excess risk is minimax optimal. Finally, we note that the theorem can be generalized in a straightforward manner to the situation when each machine has different number of examples.

\section{Proof Outline}
\label{sec:poutline}
We present here the framework for obtaining the results described in this paper; the framework has been introduced in the work of~\citet{defossez15}. Towards this purpose, we begin by introducing some notations. We begin by defining the centered estimate $\etav_t$ as:
\vspace{-0.2cm}
\begin{align*}
\etav_t \defeq \w_t - \w^*.
\end{align*}
Mini-batch SGD (with a batch size $b$) moves $\etav_{t-1}$ to $\etav_t$ using the following update:
\begin{align*}
\etav_t &= \bigg(\eye - \frac{\gamma}{b}\cdot\sum_{i=1}^b \x_{ti}\otimes\x_{ti}\bigg)\etav_{t-1} + \frac{\gamma}{b}\sum_{i=1}^b \epsilon_{ti}\x_{ti}= (\eye-\gamma\widehat{\H}_{tb})\etav_{t-1} + \gamma\cdot \xiv_{tb},
\end{align*}
where, $\widehat{\H}_{tb} = \frac{1}{b}\sum_{i=1}^b\x_{ti}\otimes\x_{ti}$ and $\xiv_{tb} = \frac{1}{b}\sum_{i=1}^b \epsilon_{ti}\x_{ti}$. Next, the tail-averaged iterate $\bar{\x}_{s,n}$ is associated with its own centered estimate $\etavb_{s,n}=\frac{1}{n-s}\sum_{i=s+1}^n\etav_i$. The analysis proceeds by tracking the covariance of the centered estimates $\etav_t$, i.e. by tracking $\E{\etav_t\otimes\etav_t}$. 

\indent\textit{Bias-Variance decomposition:} The main results of this paper are derived by going through the bias-variance decomposition, which is well known in the context of Stochastic Approximation~\citep{bach11,bach13,frostig15a}. The bias-variance decomposition allows for us to bound the generalization error by analyzing two sub-problems, namely, (i) The \emph{bias} sub-problem, which analyzes the noiseless/realizable (or the consistent linear system) problem, by setting the noise $\epsilon_{ti}=0\ \forall\ t,i$, $\etav_0^{\textrm{bias}}=\etav_0$ and (ii) the \emph{variance} sub-problem, which involves starting at the solution, i.e., $\etav_0^{\textrm{variance}} = 0$ and allowing the noise $\epsilon_{ti}$ to drive the resulting process. The corresponding tail-averaged iterates are associated with their centered estimates $\bar{\etav}_{s,n}^{\textrm{bias}}$ and $\bar{\etav}_{s,n}^{\textrm{variance}}$ respectively. The bias-variance decomposition for the square loss establishes the following relation:
\begin{align}\label{eqn:bound-covar1}
\E{\bar{\etav}_{s,n} \otimes \bar{\etav}_{s,n}} \preceq 2\cdot\bigg( \E{\bar{\etav}_{s,n}^{\text{bias}} \otimes \bar{\etav}_{s,n}^{\text{bias}}} + \E{\bar{\etav}_{s,n}^{\text{variance}} \otimes \bar{\etav}_{s,n}^{\text{variance}}}\bigg).
\end{align}
Using the bias-variance decomposition, we obtain an estimate of the generalization error as
\begin{align*}
\E{L(\bar{\x}_{s,n})}-L(\x^*) &= \tfrac{1}{2}\cdot\iprod{\H}{\E{\bar{\etav}_{s,n} \otimes \bar{\etav}_{s,n}}} \\
&\leq \trace{\H\cdot\E{\bar{\etav}_{s,n}^{\text{bias}} \otimes \bar{\etav}_{s,n}^{\text{bias}}}} + \trace{\H\cdot\E{\bar{\etav}_{s,n}^{\text{variance}} \otimes \bar{\etav}_{s,n}^{\text{variance}}}}.
\end{align*}

We now provide a few lemmas that help us bound the behavior of the bias and variance error. 
\begin{lemma}\label{lem:biasMBSGD}
	With a batch size $b$, step size $\gamma = \gammabmax/2$, the centered bias estimate $\etav_t^{\text{bias}}$ exhibits the following per step contraction:
	\begin{align*}
	\iprod{\eye}{\E{\etav_t^{\text{bias}}\otimes\etav_t^{\text{bias}}}} \leq c_{\kappa_b}\iprod{\eye}{\E{\etav_{t-1}^{\text{bias}}\otimes\etav_{t-1}^{\text{bias}}}},
	\end{align*}
	where, $c_{\kappa_b} = 1 - 1/\cnH_b$, where $\cnH_b = \frac{\infbound \cdot\rhom+ (b-1) \|\H\|_2}{b\mu}$.
\end{lemma}
\noindent Lemma~\eqref{lem:biasMBSGD} ensures that the bias decays at a geometric rate during the burn-in iterations when the iterates are not averaged; this rate holds only when the excess risk is larger than the noise level $\sigma^2$. 

We now turn to bounding the variance error. It turns out that it suffices to understand the behavior of limiting centered variance $\E{\etav_\infty^{\text{variance}}\otimes\etav_{\infty}^{\text{variance}}}$.\begin{lemma}\label{lem:varMBSGD}
	Consider the well-specified case of the streaming LSR problem. With a batch size $b$, step size $\gamma=\gammabmax/2$, the limiting centered variance $\etav_\infty^{\text{variance}}$ has an expected covariance that is upper bounded in a psd sense as:
	\begin{align*}
	\E{\etav_\infty^{\text{variance}}\otimes\etav_{\infty}^{\text{variance}}}\preceq \frac{1}{\infbound+(b-1)\|\H\|_2}\cdot\sigma^2\cdot\eye.
	\end{align*}
\end{lemma}
Characterizing the behavior of the final iterate is crucial towards obtaining bounds on the behavior of the tail-averaged iterate. In particular, the final iterate having a excess variance risk $\mathcal{O}(\sigma^2)$ (as is the case with lemma~\eqref{lem:varMBSGD}) appears crucial towards achieving minimax rates of the averaged iterate.

\section{Experimental Simulations}
\label{sec:experiments}
We conduct experiments using a synthetic example to illustrate the implications of our theoretical results on mini-batching and tail-averaging. The data is sampled from a $50-$ dimensional Gaussian with eigenvalues decaying as $\{\frac{1}{k}\}_{k=1}^{50}$ (condition number $\kappa=50$), and the variance $\sigma^2$ of the (additive noise) noise is $0.01$. In this case, our estimated batch size according to Theorem~\ref{thm:mainMBTANR} is $\bthresh=11$. Our results are presented by averaging over $100$ independent runs of the Algorithm, and each run employs $200\cnH$ samples. All plots are log-log with x-axis being the depth, and y-axis the excess risk. For our plots, we assume that each iteration takes constant time for all batch sizes; this is done to present evidence regarding the tightness of our mini-batching characterization limits that yield linear parallelization speedups over standard serial SGD.

\begin{figure}[t!] 
	\begin{subfigure}{0.33\textwidth}
		\includegraphics[width=\linewidth,height=4cm]{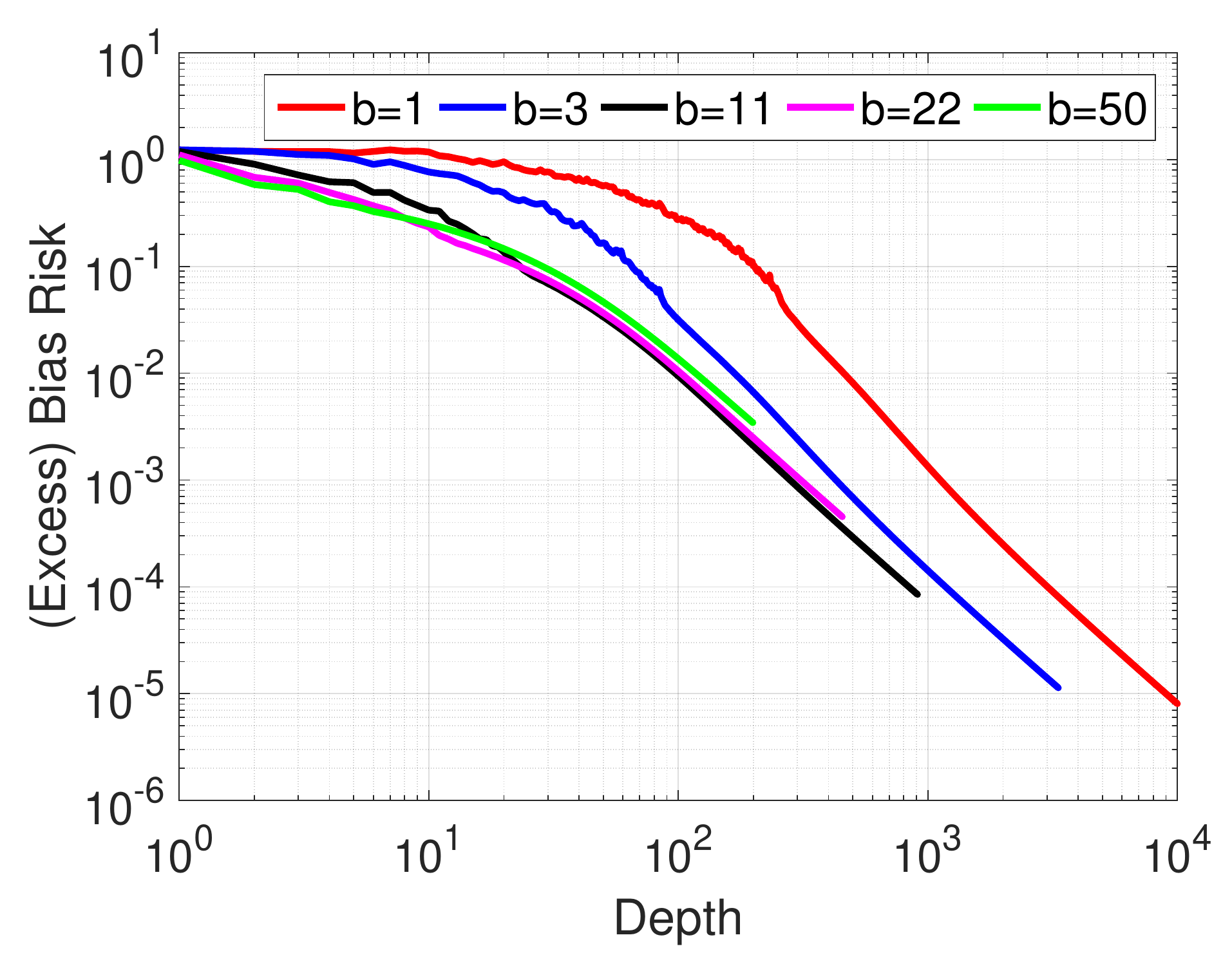}
		\caption{Bias Risk} \label{fig:1a}
	\end{subfigure}
	\begin{subfigure}{0.33\textwidth}
		\includegraphics[width=\linewidth,height=4cm]{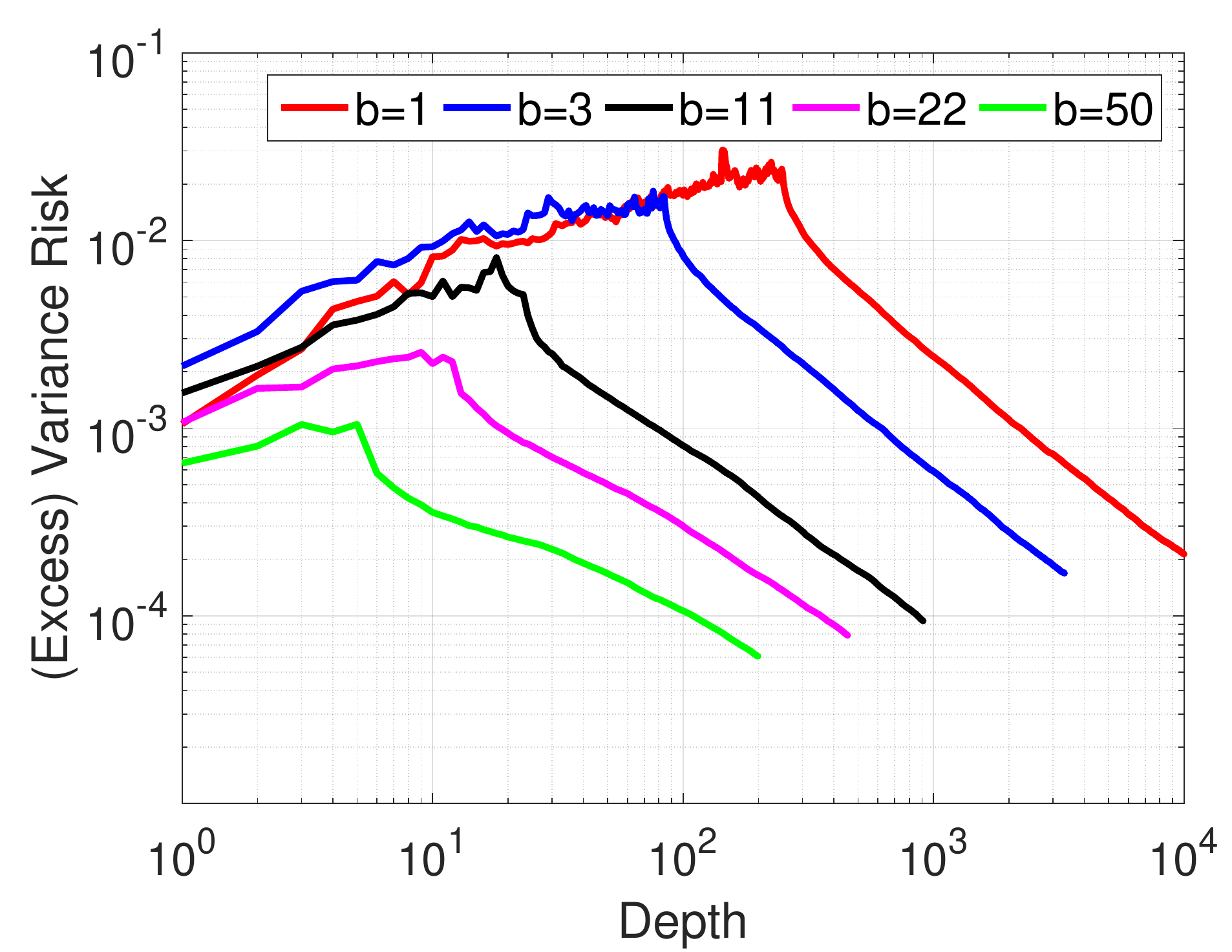}
		\caption{Variance Risk} \label{fig:1b}
	\end{subfigure}
	\begin{subfigure}{0.33\textwidth}
		\includegraphics[width=\linewidth,height=4cm]{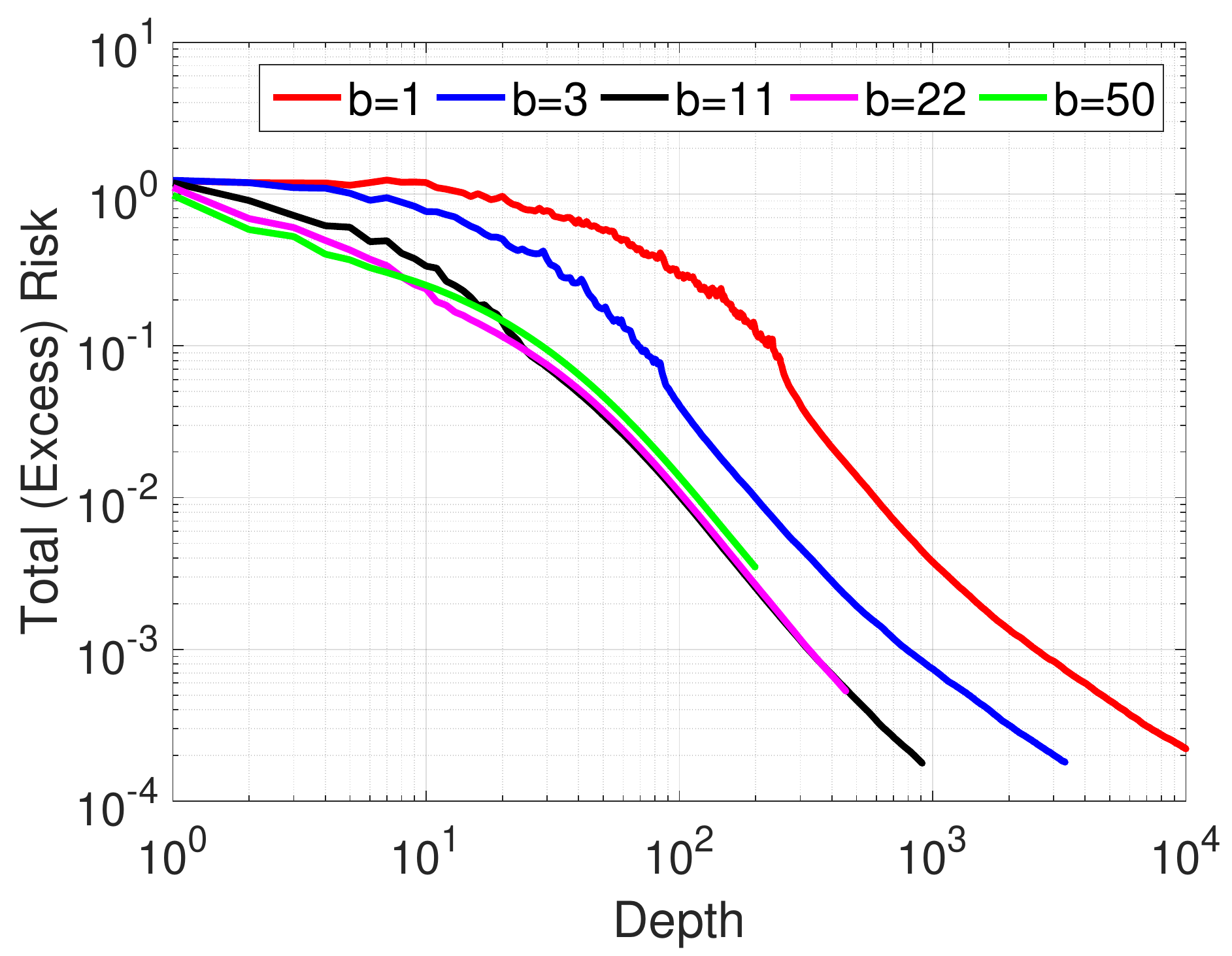}
		\caption{Total Risk} \label{fig:1c}
	\end{subfigure}
	\caption{Effect of increased batch sizes on the Algorithm's generalization error. The variance decreases monotonically with increasing batch size. The bias indicates that the rate of decay increases till the optimal $b_{thresh}$. With $b=\bthresh$, mini-batch SGD obtains the same generalization error as batchsize $1$ using smaller number of iterations (i.e. smaller depth) compared to larger batch sizes.} \label{fig:batchSize}
\end{figure}

We consider the effect of mini-batching (in figure~\ref{fig:batchSize}) with batch sizes of $1$, $3$, $\bthresh=11$, $2\cdot\bthresh=22$ and $d=50$. Averaging begins after observing a fixed number of samples (set as $5\kappa$). We see that the rate of bias decay (figure~\ref{fig:1a}) increases until reaching a mini-batch size of $\bthresh$, saturating thereafter; this implies we are inefficient in terms of sample size. As expected, the rate of decay of variance (figure~\ref{fig:1b}) is monotonic as a function of mini-batch size. Finally, the overall error (figure~\ref{fig:1c}) shows the tightness of our mini-batching characterization: with a batch size of $\bthresh$, we obtain a generalization error that is the {\em same} as using batch size of $1$ with the number of (serial) iterations (i.e. depth) that is an order of magnitude smaller. Subsequently, we note that larger batch sizes worsen generalization error thus depicting the tightness of our characterization of $\bthresh$.

In the next experiment, we fix batch size $=\bthresh$ and consider the effect of when tail-averaging begins (figure~\ref{fig:averaging}). We consider averaging iterates from the start (as prescribed by~\cite{defossez15}), after a quarter/half of total number of iterations, and unaveraged SGD as well. We see that the bias (figure~\ref{fig:2a}) exhibits a geometric decay in the unaveraged phase while switching to an slower $\mathcal{O}(\frac{1}{t^2})$ rate with averaging. The variance (figure~\ref{fig:2b}) tends to increase and stabilize at $\mathcal{O}(\frac{\sigma^2}{\bthresh})$ in the absence of averaging, while switching to a $\mathcal{O}(\frac{1}{N})$ decay rate when averaging begins. The overall generalization error (figure~\ref{fig:2c}) shows the superiority of the scheme where averaging after a burn-in period allows the bias to decay towards the noise level at a geometric rate, following which tail-averaging allows us to decay the variance term, providing credence to our theoretical results that tail-averaged SGD allows us to obtain better generalization error as a function of sample size.
\vspace*{-2mm}
\begin{figure}[t!]
	\begin{subfigure}{0.33\textwidth}
		\includegraphics[width=\linewidth,height=4cm]{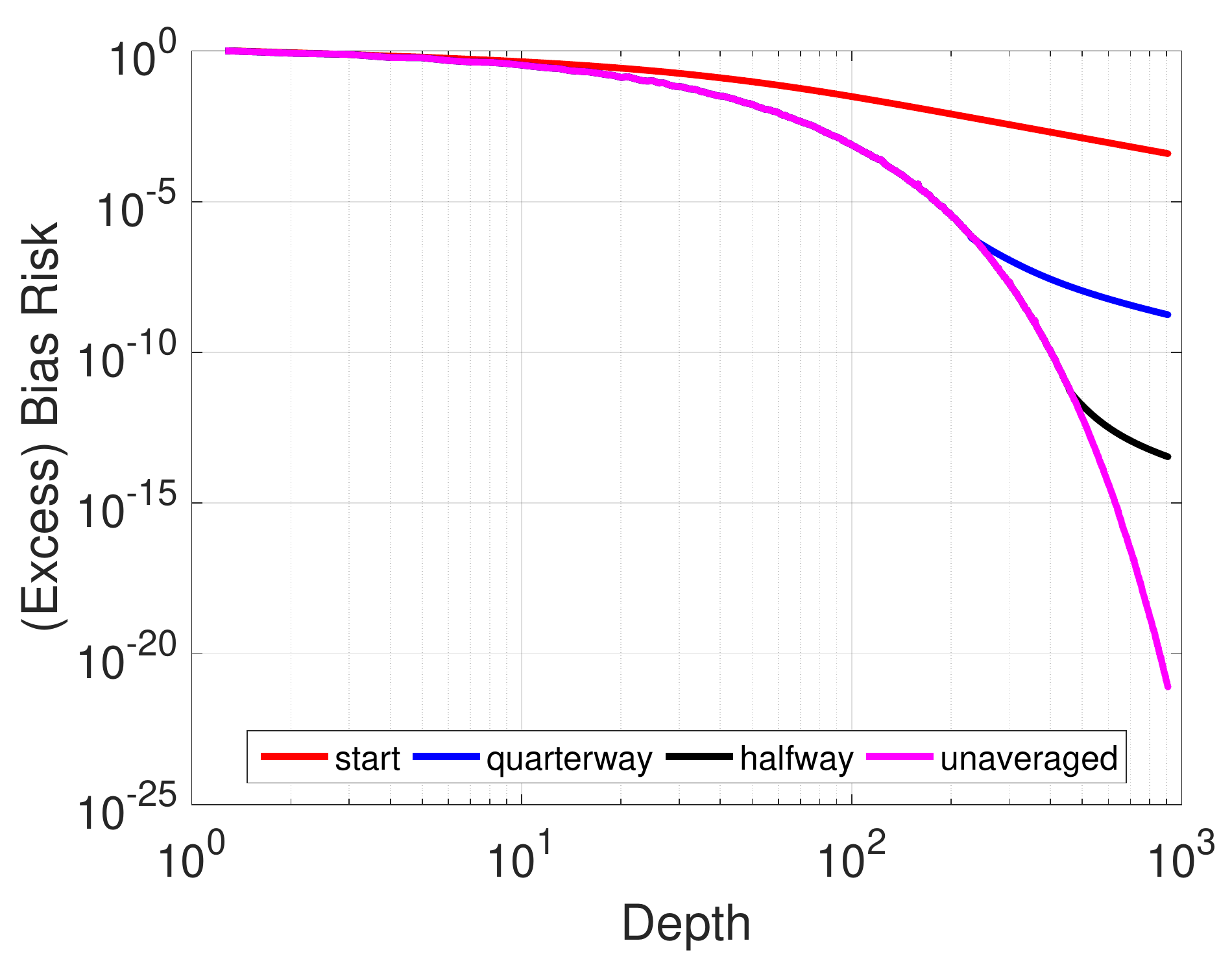}
		\caption{Bias Risk} \label{fig:2a}
	\end{subfigure}
	\begin{subfigure}{0.33\textwidth}
		\includegraphics[width=\linewidth,height=4cm]{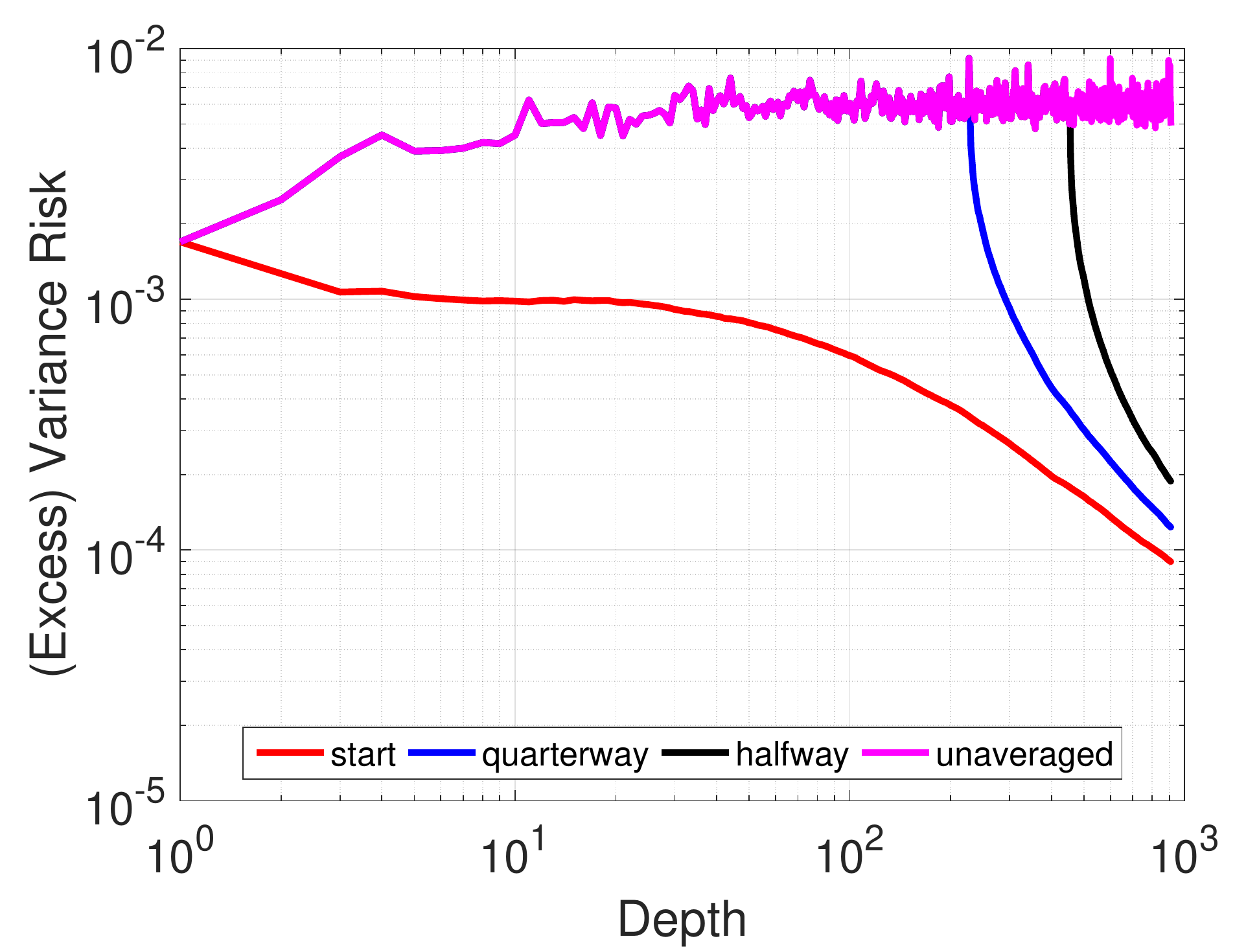}
		\caption{Variance Risk} \label{fig:2b}
	\end{subfigure}
	\begin{subfigure}{0.33\textwidth}
		\includegraphics[width=\linewidth,height=4cm]{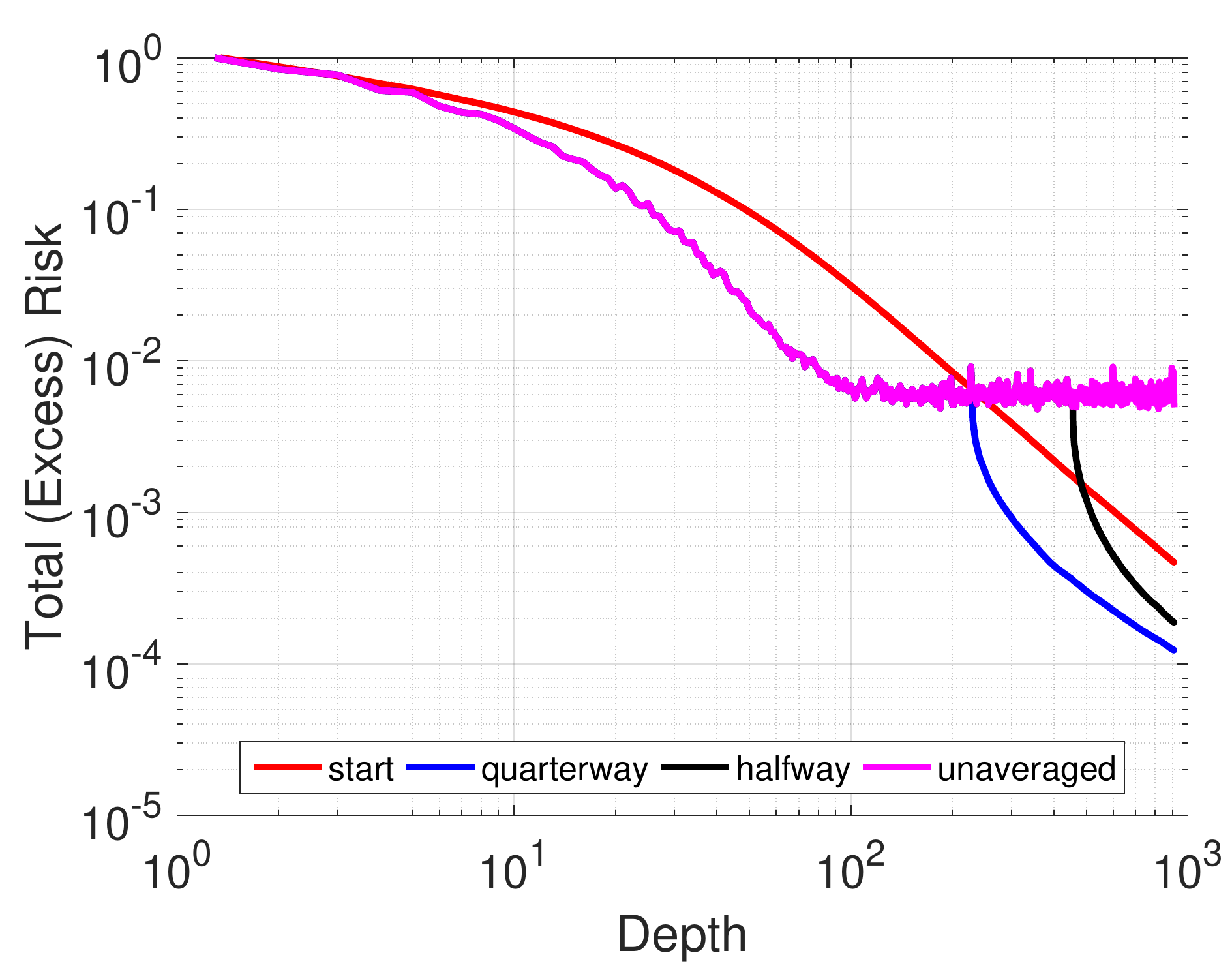}
		\caption{Total Risk} \label{fig:2c}
	\end{subfigure}
	\caption{[Zoom in to see detail] Effect of tail-averaging with mini-batch size of $b_{thresh}=11$.} \label{fig:averaging}
\end{figure}


\section{Concluding Remarks}
\label{sec:conclusion}
This paper analyzes several algorithmic primitives often used in practice in conjunction with vanilla SGD for the stochastic approximation problem. In particular, this paper provides a sharp non-asymptotic treatment of (a) mini-batching, (b) tail-averaging, (c) effects of model mismatch, (d) behaviour of the final iterate, (e) highly parallel SGD method based on doubling batch sizes and (f) model-averaging/parameter mixing schemes for the strongly convex streaming LSR problem.

The effect of mini-batching and other algorithmic primitives mentioned above can be understood for a variety of models and/or algorithms. In particular, future directions could include understanding these issues for stochastic approximation with the Logistic Loss~\citep{bach14}, streaming PCA~\citep{jain16}, and other algorithms such as streaming SVRG~\citep{frostig15a}.

\paragraph{Acknowledgements}Sham Kakade acknowledges funding from Washington Research Foundation Fund for Innovation in Data-Intensive Discovery and National Science Foundation (NSF) through awards CCF-1703574 and CCF-1740551. Rahul Kidambi thanks James Saunderson for useful discussions on matrix operator theory.

\newpage


\begin{thebibliography}{55}
	\providecommand{\natexlab}[1]{#1}
	\providecommand{\url}[1]{\texttt{#1}}
	\expandafter\ifx\csname urlstyle\endcsname\relax
	\providecommand{\doi}[1]{doi: #1}\else
	\providecommand{\doi}{doi: \begingroup \urlstyle{rm}\Url}\fi
	
	\bibitem[Agarwal et~al.(2012)Agarwal, Bartlett, Ravikumar, and
	Wainwright]{agarwal12}
	A.~Agarwal, P.~L. Bartlett, P.~Ravikumar, and M.~J. Wainwright.
	\newblock Information-theoretic lower bounds on the oracle complexity of
	stochastic convex optimization.
	\newblock \emph{IEEE Transactions on Information Theory}, 2012.
	
	\bibitem[Allen-Zhu(2016)]{Zhu16}
	Z.~Allen-Zhu.
	\newblock Katyusha: The first direct acceleration of stochastic gradient
	methods.
	\newblock \emph{CoRR}, abs/1603.05953, 2016.
	
	\bibitem[Anbar(1971)]{anbar1971optimal}
	D.~Anbar.
	\newblock \emph{On Optimal Estimation Methods Using Stochastic Approximation
		Procedures}.
	\newblock University of California, 1971.
	
	\bibitem[Bach and Moulines(2011)]{bach11}
	F.~Bach and E.~Moulines.
	\newblock Non-asymptotic analysis of stochastic approximation algorithms for
	machine learning.
	\newblock In \emph{Neural Information Processing Systems (NIPS) 24}, 2011.
	
	\bibitem[Bach(2014)]{bach14}
	F.~R. Bach.
	\newblock Adaptivity of averaged stochastic gradient descent to local strong
	convexity for logistic regression.
	\newblock \emph{In Journal of Machine Learning Research (JMLR)}, volume 15,
	2014.
	
	\bibitem[Bach and Moulines(2013)]{bach13}
	F.~R. Bach and E.~Moulines.
	\newblock Non-strongly-convex smooth stochastic approximation with convergence
	rate {O}(1/n).
	\newblock In \emph{Neural Information Processing Systems (NIPS) 26}, 2013.
	
	\bibitem[Bhatia(2007)]{bhatia2007}
	R.~Bhatia.
	\newblock \emph{Positive Definite Matrices}.
	\newblock Princeton Series in Applied Mathematics. Princeton University Press,
	2007.
	
	\bibitem[Bottou and Bousquet(2007)]{bottou07}
	L.~Bottou and O.~Bousquet.
	\newblock The tradeoffs of large scale learning.
	\newblock In \emph{Neural Information Processing Systems (NIPS) 20}, 2007.
	
	\bibitem[Bottou et~al.(2016)Bottou, Curtis, and
	Nocedal]{bottou2016optimization}
	L.~Bottou, F.~E. Curtis, and J.~Nocedal.
	\newblock Optimization methods for large-scale machine learning.
	\newblock \emph{arXiv preprint arXiv:1606.04838}, 2016.
	
	\bibitem[Bradley et~al.(2011)Bradley, Kyrola, Bickson, and Guestrin]{bradley11}
	J.~K. Bradley, A.~Kyrola, D.~Bickson, and C.~Guestrin.
	\newblock Parallel coordinate descent for l1-regularized loss minimization.
	\newblock In \emph{International Conference on Machine Learning (ICML)}, 2011.
	
	\bibitem[Cauchy(1847)]{cauchy1847}
	L.~A. Cauchy.
	\newblock M\'ethode g\'en\'erale pour la r\'esolution des syst\'emes
	d'\'equations simultanees.
	\newblock \emph{C. R. Acad. Sci. Paris}, 1847.
	
	\bibitem[Cotter et~al.(2011)Cotter, Shamir, Srebro, and Sridharan]{cotter11}
	A.~Cotter, O.~Shamir, N.~Srebro, and K.~Sridharan.
	\newblock Better mini-batch algorithms via accelerated gradient methods.
	\newblock In \emph{Neural Information Processing Systems (NIPS) 24}, 2011.
	
	\bibitem[Defazio(2016)]{Defazio16}
	A.~Defazio.
	\newblock A simple practical accelerated method for finite sums.
	\newblock In \emph{Neural Information Processing Systems (NIPS) 29}, 2016.
	
	\bibitem[Defazio et~al.(2014)Defazio, Bach, and Lacoste-Julien]{defazio14}
	A.~Defazio, F.~R. Bach, and S.~Lacoste-Julien.
	\newblock {SAGA}: {A} fast incremental gradient method with support for
	non-strongly convex composite objectives.
	\newblock In \emph{Neural Information Processing Systems (NIPS) 27}, 2014.
	
	\bibitem[D{\'e}fossez and Bach(2015)]{defossez15}
	A.~D{\'e}fossez and F.~R. Bach.
	\newblock Averaged least-mean-squares: Bias-variance trade-offs and optimal
	sampling distributions.
	\newblock In \emph{Artifical Intelligence and Statistics (AISTATS)}, 2015.
	
	\bibitem[Dekel et~al.(2012)Dekel, Gilad-Bachrach, Shamir, and Xiao]{dekel12}
	O.~Dekel, R.~Gilad-Bachrach, O.~Shamir, and L.~Xiao.
	\newblock Optimal distributed online prediction using mini-batches.
	\newblock \emph{Journal of Machine Learning Research (JMLR)}, volume 13, 2012.
	
	\bibitem[Dieuleveut and Bach(2015)]{dieuleveut15}
	A.~Dieuleveut and F.~Bach.
	\newblock Non-parametric stochastic approximation with large step sizes.
	\newblock \emph{The Annals of Statistics}, 2015.
	
	\bibitem[Duchi et~al.(2015)Duchi, Chaturapruek, and R{\'e}]{duchi15}
	J.~C. Duchi, S.~Chaturapruek, and C.~R{\'e}.
	\newblock Asynchronous stochastic convex optimization.
	\newblock \emph{CoRR}, abs/1508.00882, 2015.
	
	\bibitem[Fabian(1973)]{Fabian:1973:AES}
	V.~Fabian.
	\newblock Asymptotically efficient stochastic approximation; the {RM} case.
	\newblock \emph{Annals of Statistics}, 1\penalty0 (3), 1973.
	
	\bibitem[Frostig et~al.(2015{\natexlab{a}})Frostig, Ge, Kakade, and
	Sidford]{frostig15b}
	R.~Frostig, R.~Ge, S.~Kakade, and A.~Sidford.
	\newblock Un-regularizing: approximate proximal point and faster stochastic
	algorithms for empirical risk minimization.
	\newblock In \emph{International Conference on Machine Learning (ICML)},
	2015{\natexlab{a}}.
	
	\bibitem[Frostig et~al.(2015{\natexlab{b}})Frostig, Ge, Kakade, and
	Sidford]{frostig15a}
	R.~Frostig, R.~Ge, S.~M. Kakade, and A.~Sidford.
	\newblock Competing with the empirical risk minimizer in a single pass.
	\newblock In \emph{Conference on Learning Theory (COLT)}, 2015{\natexlab{b}}.
	
	\bibitem[Goyal et~al.(2017)Goyal, Doll{\'a}r, Girshick, Noordhuis, Wesolowski,
	Kyrola, Tulloch, Jia, and He]{goyal2017accurate}
	P.~Goyal, P.~Doll{\'a}r, R.~Girshick, P.~Noordhuis, L.~Wesolowski, A.~Kyrola,
	A.~Tulloch, Y.~Jia, and K.~He.
	\newblock Accurate, large minibatch sgd: training imagenet in 1 hour.
	\newblock \emph{arXiv preprint arXiv:1706.02677}, 2017.
	
	\bibitem[Jain et~al.(2016{\natexlab{a}})Jain, Jin, Kakade, Netrapalli, and
	Sidford]{jain16}
	P.~Jain, C.~Jin, S.~M. Kakade, P.~Netrapalli, and A.~Sidford.
	\newblock Streaming pca: Matching matrix bernstein and near-optimal finite
	sample guarantees for oja's algorithm.
	\newblock In \emph{Conference on Learning Theory (COLT)}, 2016{\natexlab{a}}.
	
	\bibitem[Jain et~al.(2016{\natexlab{b}})Jain, Kakade, Kidambi, Netrapalli, and
	Sidford]{jain2016parallelizing}
	P.~Jain, S.~M. Kakade, R.~Kidambi, P.~Netrapalli, and A.~Sidford.
	\newblock Parallelizing stochastic approximation through mini-batching and
	tail-averaging.
	\newblock \emph{arXiv preprint arXiv:1610.03774}, 2016{\natexlab{b}}.
	
	\bibitem[Jain et~al.(2017{\natexlab{a}})Jain, Kakade, Kidambi, Netrapalli,
	Pillutla, and Sidford]{jain2017markov}
	P.~Jain, S.~M. Kakade, R.~Kidambi, P.~Netrapalli, V.~K. Pillutla, and
	A.~Sidford.
	\newblock A markov chain theory approach to characterizing the minimax
	optimality of stochastic gradient descent (for least squares).
	\newblock \emph{arXiv preprint arXiv:1710.09430}, 2017{\natexlab{a}}.
	
	\bibitem[Jain et~al.(2017{\natexlab{b}})Jain, Kakade, Kidambi, Netrapalli, and
	Sidford]{JainKKNS17}
	P.~Jain, S.~M. Kakade, R.~Kidambi, P.~Netrapalli, and A.~Sidford.
	\newblock Accelerating stochastic gradient descent.
	\newblock \emph{arXiv preprint arXiv:1704.08227}, 2017{\natexlab{b}}.
	
	\bibitem[Johnson and Zhang(2013)]{johnson13}
	R.~Johnson and T.~Zhang.
	\newblock Accelerating stochastic gradient descent using predictive variance
	reduction.
	\newblock In \emph{Neural Information Processing Systems (NIPS) 26}, 2013.
	
	\bibitem[Kushner and Clark(1978)]{KushnerClark}
	H.~J. Kushner and D.~S. Clark.
	\newblock \emph{Stochastic Approximation Methods for Constrained and
		Unconstrained Systems.}
	\newblock Springer-Verlag, 1978.
	
	\bibitem[Kushner and Yin(1987)]{kushner87}
	H.~J. Kushner and G.~Yin.
	\newblock Asymptotic properties of distributed and communicating stochastic
	approximation algorithms.
	\newblock \emph{SIAM Journal on Control and Optimization}, 25\penalty0
	(5):\penalty0 1266--1290, 1987.
	
	\bibitem[Kushner and Yin(2003)]{kushner03}
	H.~J. Kushner and G.~Yin.
	\newblock Stochastic approximation and recursive algorithms and applications.
	\newblock \emph{Springer-Verlag}, 2003.
	
	\bibitem[Lehmann and Casella(1998)]{lehmann1998theory}
	E.~L. Lehmann and G.~Casella.
	\newblock \emph{Theory of Point Estimation}.
	\newblock Springer Texts in Statistics. Springer, 1998.
	
	\bibitem[Li et~al.(2014)Li, Zhang, Chen, and Smola]{li14}
	M.~Li, T.~Zhang, Y.~Chen, and A.~J. Smola.
	\newblock Efficient mini-batch training for stochastic optimization.
	\newblock In \emph{Knowledge Discovery and Data Mining (KDD)}, 2014.
	
	\bibitem[Lin et~al.(2015)Lin, Mairal, and Harchaoui]{LinMH15}
	H.~Lin, J.~Mairal, and Z.~Harchaoui.
	\newblock A universal catalyst for first-order optimization.
	\newblock In \emph{Neural Information Processing Systems (NIPS)}, 2015.
	
	\bibitem[Mann et~al.(2009)Mann, McDonald, Mohri, Silberman, and Walker]{mann09}
	G.~Mann, R.~T. McDonald, M.~Mohri, N.~Silberman, and D.~Walker.
	\newblock Efficient large-scale distributed training of conditional maximum
	entropy models.
	\newblock In \emph{Neural Information Processing Systems (NIPS) 22}, 2009.
	
	\bibitem[Merity et~al.(2017)Merity, Keskar, and Socher]{merityAveraging17}
	S.~Merity, N.~S. Keskar, and R.~Socher.
	\newblock Regularizing and optimizing lstm language models.
	\newblock \emph{arXiv preprint arXiv:1708.02182}, 2017.
	
	\bibitem[Needell et~al.(2016)Needell, Srebro, and Ward]{needell16}
	D.~Needell, N.~Srebro, and R.~Ward.
	\newblock Stochastic gradient descent, weighted sampling, and the randomized
	kaczmarz algorithm.
	\newblock \emph{Mathematical Programming}, volume 155, 2016.
	
	\bibitem[Nemirovsky and Yudin(1983)]{nemirovsky83}
	A.~S. Nemirovsky and D.~B. Yudin.
	\newblock \emph{Problem Complexity and Method Efficiency in Optimization}.
	\newblock John Wiley, 1983.
	
	\bibitem[Nesterov(1983)]{Nesterov83}
	Y.~E. Nesterov.
	\newblock A method for unconstrained convex minimization problem with the rate
	of convergence ${O}(1/k^2)$.
	\newblock \emph{Doklady AN SSSR}, 269, 1983.
	
	\bibitem[Niu et~al.(2011)Niu, Recht, Re, and Wright]{niu11}
	F.~Niu, B.~Recht, C.~Re, and S.~J. Wright.
	\newblock Hogwild: {A} lock-free approach to parallelizing stochastic gradient
	descent.
	\newblock In \emph{Neural Information Processing Systems (NIPS) 24}, 2011.
	
	\bibitem[Polyak(1964)]{Polyak64}
	B.~T. Polyak.
	\newblock Some methods of speeding up the convergence of iteration methods.
	\newblock \emph{USSR Computational Mathematics and Mathematical Physics}, 4,
	1964.
	
	\bibitem[Polyak and Juditsky(1992)]{polyak92}
	B.~T. Polyak and A.~B. Juditsky.
	\newblock Acceleration of stochastic approximation by averaging.
	\newblock \emph{SIAM J Control Optim}, volume 30, 1992.
	
	\bibitem[Robbins and Monro(1951)]{robbins51}
	H.~Robbins and S.~Monro.
	\newblock A stochastic approximation method.
	\newblock \emph{Ann. Math. Stat.}, vol. 22, 1951.
	
	\bibitem[Rosenblatt and Nadler(2014)]{rosenblatt14}
	J.~Rosenblatt and B.~Nadler.
	\newblock On the optimality of averaging in distributed statistical learning.
	\newblock \emph{CoRR}, abs/1407.2724, 2014.
	
	\bibitem[Roux et~al.(2012)Roux, Schmidt, and Bach]{roux12}
	N.~L. Roux, M.~Schmidt, and F.~R. Bach.
	\newblock A stochastic gradient method with an exponential convergence rate for
	strongly-convex optimization with finite training sets.
	\newblock In \emph{Neural Information Processing Systems (NIPS) 25}, 2012.
	
	\bibitem[Ruppert(1988)]{ruppert88}
	D.~Ruppert.
	\newblock Efficient estimations from a slowly convergent robbins-monro process.
	\newblock \emph{Tech. Report, ORIE, Cornell University}, 1988.
	
	\bibitem[Shalev-Shwartz and Zhang(2012)]{shwartz12}
	S.~Shalev-Shwartz and T.~Zhang.
	\newblock Stochastic dual coordinate ascent methods for regularized loss
	minimization.
	\newblock \emph{CoRR}, abs/1209.1873, 2012.
	
	\bibitem[Shalev-Shwartz and Zhang(2013{\natexlab{a}})]{ShwartzZ13}
	S.~Shalev-Shwartz and T.~Zhang.
	\newblock Accelerated mini-batch stochastic dual coordinate ascent.
	\newblock In \emph{Neural Information Processing Systems (NIPS) 26},
	2013{\natexlab{a}}.
	
	\bibitem[Shalev-Shwartz and Zhang(2013{\natexlab{b}})]{shwartz13}
	S.~Shalev-Shwartz and T.~Zhang.
	\newblock Accelerated mini-batch stochastic dual coordinate ascent.
	\newblock In \emph{Neural Information Processing Systems (NIPS) 26},
	2013{\natexlab{b}}.
	
	\bibitem[Smith et~al.(2017)Smith, Kindermans, and Le]{smith2017don}
	S.~L. Smith, P.-J. Kindermans, and Q.~V. Le.
	\newblock Don't decay the learning rate, increase the batch size.
	\newblock \emph{arXiv preprint arXiv:1711.00489}, 2017.
	
	\bibitem[Tak{\'a}c et~al.(2013)Tak{\'a}c, Bijral, Richt{\'a}rik, and
	Srebro]{takac13}
	M.~Tak{\'a}c, A.~S. Bijral, P.~Richt{\'a}rik, and N.~Srebro.
	\newblock Mini-batch primal and dual methods for {SVM}s.
	\newblock In \emph{International Conference on Machine Learning (ICML)},
	volume~28, 2013.
	
	\bibitem[Tak{\'a}c et~al.(2015)Tak{\'a}c, Richt{\'a}rik, and Srebro]{takac15}
	M.~Tak{\'a}c, P.~Richt{\'a}rik, and N.~Srebro.
	\newblock Distributed mini-batch sdca.
	\newblock \emph{CoRR}, abs/1507.08322, 2015.
	
	\bibitem[van~der Vaart(2000)]{Vaart00}
	A.~W. van~der Vaart.
	\newblock \emph{Asymptotic Statistics}.
	\newblock Cambridge University Publishers, 2000.
	
	\bibitem[Zhang and Xiao(2015)]{zhang15b}
	Y.~Zhang and L.~Xiao.
	\newblock Disco: Distributed optimization for self-concordant empirical loss.
	\newblock In \emph{International Conference on Machine Learning (ICML)}, 2015.
	
	\bibitem[Zhang et~al.(2015)Zhang, Duchi, and Wainwright]{zhang15a}
	Y.~Zhang, J.~C. Duchi, and M.~Wainwright.
	\newblock Divide and conquer ridge regression: A distributed algorithm with
	minimax optimal rates.
	\newblock \emph{Journal of Machine Learning Research (JMLR)}, volume 16, 2015.
	
	\bibitem[Zinkevich et~al.(2011)Zinkevich, Smola, Weimer, and Li]{zinkevich11}
	M.~A. Zinkevich, A.~Smola, M.~Weimer, and L.~Li.
	\newblock Parallelized stochastic gradient descent.
	\newblock In \emph{Neural Information Processing Systems (NIPS) 24}, 2011.
	
\end{thebibliography}

\newpage

\appendix

\section{Appendix}
\label{sec:theorems}
We begin with a note on the organization:
\begin{itemize}
\item Section~\ref{sec:mainNotations} introduces notations necessary for the rest of the appendix.
\item Section~\ref{ssection:a1} derives the mini-batch SGD update and provides the bias-variance decomposition and reasons about its implication in bounding the generalization error.
\item Section~\ref{sssection:biasErrorLemmas} provides lemmas that are used to bound the bias error.
\item Section~\ref{sssection:varianceErrorLemmas} provides lemmas that are used to bound the variance error.
\item Section~\ref{ssec:mainRes} uses the results of the previous sections to obtain the main results of this paper.
\end{itemize}

\subsection{Notations}\label{sec:mainNotations}
We begin by introducing the centered iterate $\etav_t$ i.e.:
\begin{align*}
\etav_t \defeq \w_t - \ws.
\end{align*}
In a manner similar to $\w_t$, the tail-averaged iterate $\wbar_{t,N}$ is associated with its corresponding centered estimate $\etavb_{t,N}\defeq\wbar_{t,N}-\ws=\frac{1}{N}\sum_{s=t}^{t+N-1}(\w_s-\ws)=\frac{1}{N}\sum_{s=t}^{t+N-1}\etav_s$.
Next, let $\phiv_t$ denote the expected covariance of the centered estimate $\etav_t$, i.e.
\begin{align*}
\phiv_t \defeq \E{\etav_t\otimes\etav_t},
\end{align*}
and in a similar way as the final iterate $\w_t$, the tail-averaged estimate $\wbar_{t,N}$ is associated with its expected covariance, i.e. $\bar{\phiv}_{t,N}\defeq\E{\etavb_{t,N}\otimes\etavb_{t,N}}$.

\subsection{Mini-Batch Tail-Averaged SGD: Bias-Variance Decomposition}\label{ssection:a1}
In section~\ref{sssection:baserecursion}, we derive the basic recursion governing the evolution of the iterates $\w_t$ and the tail-averaged iterate $\wbar_{s+1,N}$. In section~\ref{sssection:biasVarDecompFinalIterate} we provide the bias-variance decomposition of the final iterate. 
In section~\ref{sssection:biasVarDecompTailAvgIterate}, we provide the bias-variance decomposition of the tail-averaged iterate.
\subsubsection{The basic recursion}
\label{sssection:baserecursion}
At each iteration $t$ of Algorithm~\ref{alg:mbSGD}, we are provided with $b$ fresh samples $\{(\x_{ti},y_{ti})\}_{i=1}^b$ drawn i.i.d. from the distribution $\mathcal{D}$. We start by recounting the mini-batch gradient descent update rule that allows us to move from iterate $\w_{t-1}$ to $\w_{t}$:
\begin{equation*}
\w_{t} = \w_{t-1} - \gb \sum_{i=1}^b (\langle \w_{t-1},\x_{ti} \rangle- y_{ti})\x_{ti}
\end{equation*}
where, $0<\gamma<\gammabmax[b]$ is the constant step size that is set to a value less than the maximum allowed learning rate $\gammabmax[b]$. We also recount the definition of $\wbar_{t,N}$ which is the iterate obtained by averaging for $N$ iterations starting from the $t^{th}$ iteration, i.e.,
\begin{equation*}
\wbar_{t,N}=\frac{1}{N}\sum_{s=t}^{t+N-1} \w_s
\end{equation*}
Let us first denote the residual error term by $\epsilon_i=y_i-\iprod {\ws}{\x_i}$. By the first order optimality conditions of $\ws$, we observe that $\epsilon$ and $\x$ are orthogonal, i.e, $\mathbb{E}_{(\x,y)\sim\mathcal{D}}[\epsilon\cdot\x]=0$. 
For any estimate $\w$, the excess risk/generalization error can be written as:
\begin{equation}
\label{eq:genErrorSqLoss}
L(\w)-L(\ws)=\frac{1}{2}\Tr\bigg(\H\cdot\big(\etav\otimes\etav\big)\bigg), \text{ with } \etav=\w-\ws.
\end{equation}
We now write out the main recursion governing the mini-batch SGD updates in terms of $\etav_{.}$:
\begin{align}
\label{eq:recursion}
\etav_t &= \Big(\eye-\gb \sum_{i=1}^b \x_{ti}\otimes \x_{ti}\Big)\etav_{t-1} + \gb\sum_{i=1}^b \epsilon_{ti}\x_{ti} \nonumber\\
&= \Big(\eye-\gb \sum_{i=1}^b \x_{ti}\otimes \x_{ti}\Big)\etav_{t-1} + \gb\sum_{i=1}^b \xiv_{ti} \nonumber\\
&= \PP_{tb}\etav_{t-1} + \gamma \zetav_{tb}
\end{align}
Where, $\PP_{tb}\defeq\Big(\eye-\gb \sum_{i=1}^b \x_{ti}\otimes \x_{ti}\Big)$ and $\zetav_{tb}\defeq\frac{1}{b}\sum_{i=1}^b \xiv_{ti}=\frac{1}{b}\sum_{i=1}^b \epsilon_{ti}\x_{ti}$. Equation~\ref{eq:recursion} automatically brings out the ``operator'' view of analyzing the (expected) covariance of the centered estimate $\phiv_t=\E{\etav_t\otimes\etav_t}$ to provide an estimate of the generalization error. We now note the following about the covariance of $\zetav_{tb}$:
\begin{align}
\label{eq:perIterCovariance}
\mathbb{E}[\zetav_{tb}\otimes\zetav_{t' b}] &= \frac{1}{b^2}\sum_{i,j}\mathbb{E}[\xiv_{ti}\otimes\xiv_{t' j}]\nonumber\\
&=\Big[\frac{1}{b^2}\sum_{i=1}^b\mathbb{E}[\xiv_{ti}\otimes\xiv_{ti}] \Big]\mathbbm{1}[t=t']=\frac{1}{b}\Sig\ \ \mathbbm{1}[t=t']
\end{align}
Where, $\mathbbm{1}[.]$ is the indicator function, and equals $1$ if the argument inside $[.]$ is true and $0$ otherwise. We note that the expectation of the cross terms in equation~\ref{eq:perIterCovariance} is zero owing to independence of the samples $\{\x_{ti},y_{ti}\}_{i=1}^b$ as well as between $\{\x_{ti},y_{ti}\}_{i=1}^b,\ \{\x_{t'i},y_{t'i}\}_{i=1}^b \ \forall\ t\ne t'$ and owing to the first order optimality conditions. Owing to the invariance of $\zetav_{tb}$ on the iteration $t$, context permitting, we sometimes drop the iteration index $t$ from $\zetav_{tb}$ and simply refer to it as $\zetav_{b}$.

Next we expand out the recurrence \eqref{eq:recursion}. Let $\Q_{j,t} = (\prod_{k = j}^t \PP_{kb})^T$ with the convention that $\Q_{t',t} = \eye\ \forall \ t'>t$. With this notation we have:
\begin{align}
\label{eq:unrolledRecursion}
\etav_t &= \PP_{tb}\etav_{t-1} + \gamma \zetav_{tb} \nonumber\\
&= \PP_{tb}\PP_{t-1,b}...\PP_{1,b}\etav_{0} + \gamma \sum_{j=0}^{t-1}\{\PP_{tb}....\PP_{t-j+1,b}\}\zetav_{t-j,b} \nonumber\\
&= \Q_{1,t}\etav_0 + \gamma \sum_{j=0}^{t-1} \Q_{t-j+1,t}\zetav_{t-j,b} \nonumber\\
&= \Q_{1,t}\etav_0 + \gamma \sum_{j=1}^{t} \Q_{j+1,t}\zetav_{j,b}\nonumber\\
&= \etav_t^{\textrm{bias}} + \etav_t^{\textrm{variance}}\ ,
\end{align}
where, we note that 
\begin{align}
\label{eq:biasRec}
\etav_t^{\textrm{bias}}\defeq \Q_{1,t}\etav_0,
\end{align}
relates to understanding the behavior of SGD on the noiseless problem (i.e. $\zetav_{\cdot,\cdot}=0$ a.s.) and aims to quantify the dependence on the initial conditions. Further, 
\begin{align}
\label{eq:varianceRec}
\etav_t^{\textrm{variance}}\defeq \gamma \sum_{j=1}^{t} \Q_{j+1,t}\zetav_{j,b}
\end{align}
relates to the behavior of SGD when begun at the solution (i.e. $\etav_0=0$) and allowing the noise $\zetav_{\cdot,\cdot}$ to drive the process.

Furthermore, considering the tail-averaged iterate obtained by averaging the iterates of the SGD procedure for $N$ iterations starting from a certain number of iterations ``$s$'', i.e., we examine the quantity $\etavb_{s+1,N} = \wbar_{s+1,N}-\ws$, where $\wbar_{s+1,N} = \frac{1}{N}\sum_{t=s+1}^{s+N}\w_t$. We write out the expression for $\etavb_{s+1,N}$ starting out from equation~\ref{eq:unrolledRecursion}:
\begin{align}\label{eq:mts}
\etavb_{s+1,N} &= \frac{1}{N}\sum_{t=s+1}^{s+N} \etav_t \nonumber\\ 
&= \frac{1}{N}\sum_{t=s+1}^{s+N} \big(\etav_t^{\textrm{bias}} + \etav_t^{\textrm{variance}}\big)\qquad\qquad\qquad\text{(from equation~\ref{eq:unrolledRecursion})}\nonumber\\
&= \etavb_{s+1,N}^{\textrm{bias}} + \etavb_{s+1,N}^{\textrm{variance}}.
\end{align}

\subsubsection{The Final Iterate: Bias-Variance Decomposition}
\label{sssection:biasVarDecompFinalIterate}
The behavior of the final iterate is considered to be of great practical interest and we hope to shed light on the behavior of this final iterate and the tradeoffs between the learning rate and batch size. 
Since the generalization error of any iterate $\w_N$ obtained by running mini-batch SGD with a batch size $b$ for a total of $N$ iterations can be estimated by tracking $\E{\etav_N\otimes\etav_N}$, where, $\etav_N = \w_N-\ws$, we provide a simple psd upper bound on the outer product of interest, i.e.:
\begin{align*}
\E{\etav_N\otimes\etav_N} &= \E{\big(\etav_N^{\textrm{bias}} + \etav_N^{\textrm{variance}}\big)\otimes\big(\etav_N^{\textrm{bias}} + \etav_N^{\textrm{variance}}\big)}\quad\text{(by substituting equation~\ref{eq:unrolledRecursion})}\\
&\preceq 2\cdot\bigg( \E{\big(\etav_N^{\textrm{bias}}\otimes\etav_N^{\textrm{bias}}\big)} + \E{\big(\etav_N^{\textrm{variance}}\otimes\etav_N^{\textrm{variance}}\big)}\bigg)
\end{align*}
Using this expression, we now write out the expression for the excess risk of the final iterate:
\begin{align}\label{eq:lastPointBVDecomposition}
\E{L(\w_N)}-L(\ws) &= \frac{1}{2}\iprod{\H}{\E{\etav_N\otimes\etav_N}}\nonumber\\
&\leq \frac{1}{2}\iprod{\H}{2\cdot\big(\E{\etav_N^{\text{bias}}\otimes\etav_N^{\text{bias}}}+\E{\etav_N^{\text{variance}}\otimes\etav_N^{\text{variance}}}\big)}\nonumber\\
&\leq 2\cdot\bigg(\frac{1}{2}\iprod{\H}{\E{\etav_N^{\text{bias}}\otimes\etav_N^{\text{bias}}}}+\frac{1}{2}\iprod{\H}{\E{\etav_N^{\text{variance}}\otimes\etav_N^{\text{variance}}}}\bigg)\nonumber\\
&= 2\cdot\bigg(\big(\E{L(\w_N^{\text{bias}})}-L(\ws)\big) + \big(\E{L(\w_N^{\text{variance}})}-L(\ws)\big)\bigg).
\end{align}

\subsubsection{The Tail-Averaged Iterate: Bias-Variance Decomposition}
\label{sssection:biasVarDecompTailAvgIterate}
Now, considering the fact that the excess risk/generalization error (equation~\ref{eq:genErrorSqLoss}) involves tracking $\E{\etavb_{s+1,N}\otimes\etavb_{s+1,N}}$, we see that the quantity of interest can be bounded by considering the behavior of SGD on bias and variance sub-problem. In particular, writing out the outerproduct of equation~\ref{eq:mts}, we see the following inequality holds through a straightforward application of Cauchy-Shwartz inequality:
\begin{align}
\label{eq:cauchyShwartzBV}
\E{\etavb_{s+1,N}\otimes\etavb_{s+1,N}} &\preceq 2\cdot\big(\E{\etavb_{s+1,N}^{\textrm{bias}}\otimes\etavb_{s+1,N}^{\textrm{bias}}}+\E{\etavb_{s+1,N}^{\textrm{variance}}\otimes\etavb_{s+1,N}^{\textrm{variance}}})
\end{align}
The equation above is referred to as the bias-variance decomposition and is well known from previous work on Stochastic Approximation~\citep{bach13,frostig15a,defossez15}. This implies that an upper bound on the generalization error (equation~\ref{eq:genErrorSqLoss}) is:
\begin{align}
\label{eq:genErrorUpperBound}
L(\wbar_{s+1,N})-L(\ws)&=\frac{1}{2}\iprod{\H}{\E{\etavb_{s+1,N}\otimes\etavb_{s+1,N}}}\nonumber\\
&\leq \iprod{\H}{\E{\etavb_{s+1,N}^{\textrm{bias}}\otimes\etavb_{s+1,N}^{\textrm{bias}}}} + \iprod{\H}{\E{\etavb_{s+1,N}^{\textrm{variance}}\otimes\etavb_{s+1,N}^{\textrm{variance}}}}.
\end{align}
Here, we adopt the proof approach of~\citet{jain2017markov}. In particular,~\citet{jain2017markov} provide a clean way to simplify the expression corresponding to the tail-averaged iterate. Let us consider $\E{\etavb_{s+1,N}\otimes\etavb_{s+1,N}}$ and simplify the resulting expression: in particular,
\begin{align}
\label{eq:cleanCovarianceBound}
\E{\etavb_{s+1,N}\otimes\etavb_{s+1,N}} &= \frac{1}{N^2}\sum_{l=s+1}^{s+N}\sum_{k=s+1}^{s+N}\E{\etav_l\otimes\etav_k}\nonumber\\
&= \frac{1}{N^2} \cdot \bigg(\sum_{l\geq k}\E{\etav_l\otimes\etav_k} + \sum_{l<k}\E{\etav_l\otimes\etav_k}\bigg)\nonumber\\
&\preceq\frac{1}{N^2} \cdot \bigg(\sum_{l\geq k}\E{\etav_l\otimes\etav_k} + \sum_{l\leq k}\E{\etav_l\otimes\etav_k}\bigg)\qquad(*)\nonumber\\
&=\frac{1}{N^2}\cdot\bigg(\sum_{l\geq k}(\eye-\gamma\H)^{l-k}\E{\etav_k\otimes\etav_k} + \sum_{l\leq k}\E{\etav_l\otimes\etav_l}(\eye-\gamma\H)^{k-l}\bigg)\qquad(**)\nonumber\\
&=\frac{1}{N^2}\cdot\sum_{l\leq k}\bigg( \E{\etav_l\otimes\etav_l}(\eye-\gamma\H)^{k-l} + (\eye-\gamma\H)^{k-l}\E{\etav_l\otimes\etav_l} \bigg)\nonumber\\
&=\frac{1}{N^2}\cdot\sum_{l=s+1}^{s+N}\sum_{k=l}^{s+N}\bigg( \E{\etav_l\otimes\etav_l}(\eye-\gamma\H)^{k-l} + (\eye-\gamma\H)^{k-l}\E{\etav_l\otimes\etav_l} \bigg)\nonumber\\
&=\frac{1}{N^2}\cdot\sum_{l=s+1}^{s+N}\sum_{k=l}^{\infty}\bigg( \E{\etav_l\otimes\etav_l}(\eye-\gamma\H)^{k-l} + (\eye-\gamma\H)^{k-l}\E{\etav_l\otimes\etav_l} \bigg)\nonumber\\
&\quad-\frac{1}{N^2}\cdot\sum_{l=s+1}^{s+N}\sum_{k=s+N+1}^{\infty}\bigg( \E{\etav_l\otimes\etav_l}(\eye-\gamma\H)^{k-l} + (\eye-\gamma\H)^{k-l}\E{\etav_l\otimes\etav_l} \bigg)\nonumber\\
&=\frac{1}{N^2}\cdot\sum_{l=s+1}^{s+N}\bigg( \E{\etav_l\otimes\etav_l}(\gamma\H)^{-1} + (\gamma\H)^{-1}\E{\etav_l\otimes\etav_l} \bigg)\nonumber\\
&\quad-\frac{1}{N^2}\cdot\sum_{l=s+1}^{s+N}\sum_{k=s+N+1}^{\infty}\bigg( \E{\etav_l\otimes\etav_l}(\eye-\gamma\H)^{k-l} + (\eye-\gamma\H)^{k-l}\E{\etav_l\otimes\etav_l} \bigg)\qquad(***)
\end{align}
where, $(*)$ is a valid PSD upper bound since we add and subtract the diagonal terms $\{\E{\etav_k\otimes\etav_k}\}_{k=s+1}^{s+N}$. $(**)$ follows because of the following (assume $l> k$; the other case follows similarly):
\begin{align*}
\E{\etav_l\otimes\etav_k} &= \E{\big(\PP_{lb}\etav_{l-1} + \gamma \zetav_{lb}\big)\otimes\etav_k}\\
&=\E{\E{\big(\PP_{lb}\etav_{l-1} + \gamma \zetav_{lb}\big)\otimes\etav_k|\mathcal{F}_{l-1}}}\\
&=\E{\E{\big(\PP_{lb}\etav_{l-1} + \gamma \zetav_{lb}\big)|\mathcal{F}_{l-1}}\otimes\etav_k}\\
&=(\eye-\gamma\H)\E{\etav_{l-1}\otimes\etav_k},
\end{align*}
where, the final equation follows since $\E{\PP_{lb}|\mathcal{F}_{l-1}}=\E{\eye-\gb\sum_{i=1}^b\x_{li}\otimes\x_{li}|\mathcal{F}_{l-1}}=\eye-\gamma\H$ and $\E{\zetav_{lb}|\mathcal{F}_{l-1}}=0$ from first order optimality conditions. Recursing over $l$ yields the result. $(***)$ follows from summing a (convergent) geometric series. 

This implies that the excess risk corresponding to the bias/variance term can be obtained from equation~\ref{eq:cleanCovarianceBound} by taking an inner product with $\H$, i.e.:
\begin{align}
\label{eq:refinedUpperBound}
\iprod{\H}{\E{\etavb_{s+1,N}\otimes\etavb_{s+1,N}}}&\leq\frac{1}{N^2}\cdot\sum_{l=s+1}^{s+N}\bigg(\iprod{\H}{\E{\etav_l\otimes\etav_l}(\gamma\H)^{-1} + (\gamma\H)^{-1}\E{\etav_l\otimes\etav_l} }\bigg)\nonumber\\
&\quad-\frac{1}{N^2}\cdot\sum_{l=s+1}^{s+N}\sum_{k=s+N+1}^{\infty}\bigg(\iprod{\H}{ \E{\etav_l\otimes\etav_l}(\eye-\gamma\H)^{k-l} + (\eye-\gamma\H)^{k-l}\E{\etav_l\otimes\etav_l}} \bigg)\nonumber\\
&\leq\frac{1}{N^2}\cdot\sum_{l=s+1}^{s+N}\bigg(\iprod{\H}{\E{\etav_l\otimes\etav_l}(\gamma\H)^{-1} + (\gamma\H)^{-1}\E{\etav_l\otimes\etav_l} }\bigg)\nonumber\\
&=\frac{2}{\gamma N^2}\cdot\sum_{l=s+1}^{s+N}\Tr{\bigg(\E{\etav_l\otimes\etav_l} \bigg)}
\end{align}
The upper bound on the final line follows because each term within the summation in the second line is negative owing to the following argument. Consider say,
\begin{align*}
&\iprod{\H}{\E{\etav_l\otimes\etav_l}(\eye-\gamma\H)^{k-l} + (\eye-\gamma\H)^{k-l}\E{\etav_l\otimes\etav_l}}\\
&=2\Tr{\big[\H(\eye-\gamma\H)^{k-l}\E{\etav_l\otimes\etav_l}\big]}\geq 0.
\end{align*}
Note that $\H$ and $(\eye-\gamma\H)$ commute and both are psd, implying that $\H(\eye-\gamma\H)^{k-l}$ is PSD. Finally, the trace of the product of two PSD matrices is positive with $\H(\eye-\gamma\H)^{k-l}$ being one of these PSD matrices and $\E{\etav_l\otimes\etav_l}$ being the other, thus yielding the claimed bound in equation~\ref{eq:refinedUpperBound}.

This implies that the overall error (through equation~\ref{eq:genErrorSqLoss}) can be upperbounded as:
\begin{align}
\label{eq:genErrorNewUpperBound}
\E{L(\wbar_{s+1,N})}-L(\ws)&= \frac{1}{2}\cdot\iprod{\H}{\E{\etavb_{s+1,N}\otimes\etavb_{s+1,N}}}\nonumber\\
&\leq \frac{1}{\gamma N^2} \sum_{l=s+1}^{s+N}\Tr\big(\E{\etav_l\otimes\etav_l}\big)\nonumber\\
&\leq \frac{2}{\gamma N^2}\cdot\sum_{l=s+1}^{s+N}\bigg(\Tr{\big(\E{\etav_l^{\textrm{bias}}\otimes\etav_l^{\textrm{bias}}}\big)} + \Tr{\big(\E{\etav_l^{\textrm{variance}}\otimes\etav_l^{\textrm{variance}}}\big)}\bigg),
\end{align}
where the final line follows from equation~\ref{eq:cauchyShwartzBV}. We will now bound each of these terms to precisely characterize the excess risk of mini-batch tail-averaged SGD. We refer to the bias error of the tail-averaged iterate as the following:
\begin{align}
\label{eq:biasErrorTailAvg}
\E{L(\wbar_{s+1,N}^{\textrm{bias}})}-L(\ws)\defeq \frac{2}{\gamma N^2}\sum_{l=s+1}^{s+N}\Tr\bigg(\E{\etav_l^{\textrm{bias}}\otimes\etav_l^{\textrm{bias}}}\bigg)
\end{align}
Similarly, we refer to the variance error of the tail-averaged iterate as the following:
\begin{align}
\label{eq:varianceErrorTailAvg}
\E{L(\wbar_{s+1,N}^{\textrm{variance}})}-L(\ws)\defeq \frac{2}{\gamma N^2}\sum_{l=s+1}^{s+N}\Tr\bigg(\E{\etav_l^{\textrm{variance}}\otimes\etav_l^{\textrm{variance}}}\bigg)
\end{align}

\subsection{Lemmas for bounding the bias error}\label{sssection:biasErrorLemmas}
\begin{lemma}\label{lem:biasContract}
With $\gamma\leq\frac{\gammabmax}{2} = \frac{b}{\infbound\cdot\rhom+(b-1)\|\H\|_2}$, the following bound holds:
\begin{align*}
\bigg\|\E{(\eye-\gb\sum_{j=1}^b\x_{li}\otimes\x_{li})(\eye-\gb\sum_{j=1}^b\x_{li}\otimes\x_{li})}\bigg\|_2\leq 1-\gamma\mu.
\end{align*}
\end{lemma}
\begin{proof}
This lemma generalizes one appearing in~\citet{jain2017markov} to the mini-batch size $b$ case. Denote by $\U$ the matrix of interest and consider the following:
\begin{align*}
\U &= \E{(\eye-\gb\sum_{j=1}^b\x_{li}\otimes\x_{li})(\eye-\gb\sum_{j=1}^b\x_{li}\otimes\x_{li})}\\
&= \eye - \gamma\H - \gamma\H +\bigg(\gb\bigg)^2\cdot\bigg(b \E{\|\x\|^2\x\x^{\top}} + b(b-1)\H^2\bigg)\\
&\preceq \eye - 2\gamma\H +\frac{\gamma^2}{b}\cdot\big( \infbound\H + (b-1)\|\H\|_2 \big)\H\\
&= \eye-\gamma\H,
\end{align*}
from which a spectral norm bound implied by the lemma naturally follows.
\end{proof}

\begin{lemma}\label{lem:tailAveragedBiasError}
For any learning rate $\gamma\leq\gammabmax/2$, the bias error of the tail-averaged iterate $\wbar_{s+1,N}^{\text{bias}}$ is upper bounded as:
\begin{align*}
\E{L(\wbar_{s+1,N}^{\text{bias}})}-L(\ws)\leq\frac{2}{\gamma^2 N^2\mu^2}(1-\gamma\mu)^{s+1}\cdot\big(L(\w_0)-L(\ws)\big).
\end{align*}
\end{lemma}
\begin{proof}
Before writing out the proof of the bound in the lemma, we require to bound the per step contraction properties of an SGD update in the case of the bias error (i.e. $\zetav_\cdot=0$):
\begin{align*}
\E{\|\etav_l\|^2} &= \E{\etav_{l-1}^\top (\eye-\gb\sum_{i=1}^b\x_{li}\otimes\x_{li})(\eye-\gb\sum_{i=1}^b\x_{li}\otimes\x_{li})\etav_{l-1}}\\
&= \E{\etav_{l-1}^\top \E{(\eye-\gb\sum_{i=1}^b\x_{li}\otimes\x_{li})(\eye-\gb\sum_{i=1}^b\x_{li}\otimes\x_{li})\bigg|\mathcal{F}_{l-1}}\etav_{l-1}}\\
&\leq (1-\gamma\mu)\E{\|\etav_{l-1}\|^2}\quad\text{(using lemma~\ref{lem:biasContract})}.
\end{align*}
This implies that a recursive application of the above bound yields $\E{\|\etav_l\|^2}\leq(1-\gamma\mu)^l\E{\|\etav_0\|^2}$.

Next, we consider the bias error from equation~\ref{eq:biasErrorTailAvg}:
\begin{align*}
\E{L(\wbar_{s+1,N}^{\text{bias}})}-L(\ws)&= \frac{2}{\gamma N^2}\sum_{t=s+1}^{s+N}\E{\|\etav_t\|^2}\\
&\leq \frac{2}{\gamma N^2}\sum_{t=s+1}^{\infty}\E{\|\etav_t\|^2}\\
&\leq \frac{2}{\gamma N^2}\sum_{t=s+1}^\infty (1-\gamma\mu)^t \|\etav_0\|^2\\
&= \frac{2}{\gamma N^2} (\gamma \mu)^{-1} (1-\gamma\mu)^{s+1} \|\etav_0\|^2\\
&=\frac{2}{\gamma^2\mu N^2} (1-\gamma\mu)^{s+1} \|\etav_0\|^2\\
&=\frac{2}{\gamma^2\mu^2 N^2} (1-\gamma\mu)^{s+1} \cdot \bigg(\mu\cdot\|\etav_0\|^2\bigg)\\
&\leq\frac{2}{\gamma^2\mu^2 N^2} (1-\gamma\mu)^{s+1} \cdot \bigg(L(\w_0)-L(\ws)\bigg),
\end{align*}
where in the final line, we use the fact that $\mu \eye\preceq \H$. This proves the claimed bound.
\end{proof}

\begin{lemma}\label{lem:lastPointBiasErrorBound}
For any learning rate $\gamma\leq\gammabmax/2$, the bias error of the {\bf final} iterate $\w_{N}^{\text{bias}}$ is upper bounded as:
\begin{align*}
\E{L(\w_{N}^{\text{bias}})}-L(\ws)\leq\frac{\cnH}{2}\cdot(1-\gamma\mu)^{N}\cdot\big(L(\w_0)-L(\ws)\big).
\end{align*}
\end{lemma}
\begin{proof}
Similar to the tail-averaged case, we require to bound the per step contraction properties of an SGD update in the case of the bias error (i.e. $\zetav_\cdot=0$):
\begin{align*}
\E{\|\etav_N\|^2} &= \E{\etav_{N-1}^\top (\eye-\gb\sum_{i=1}^b\x_{Ni}\otimes\x_{Ni})(\eye-\gb\sum_{i=1}^b\x_{Ni}\otimes\x_{Ni})\etav_{N-1}}\\
&= \E{\etav_{N-1}^\top \E{(\eye-\gb\sum_{i=1}^b\x_{Ni}\otimes\x_{Ni})(\eye-\gb\sum_{i=1}^b\x_{Ni}\otimes\x_{Ni})\bigg|\mathcal{F}_{N-1}}\etav_{N-1}}\\
&\leq (1-\gamma\mu)\E{\|\etav_{N-1}\|^2}\quad\text{(using lemma~\ref{lem:biasContract})}.
\end{align*}
This implies that a recursive application of the above bound yields $\E{\|\etav_N\|^2}\leq(1-\gamma\mu)^N\E{\|\etav_0\|^2}$. Then,
\begin{align*}
\E{L(\w_N^{\text{bias}})}-L(\ws) &= \frac{1}{2}\Tr{\big((\etav_N^{\textrm{bias}})^{\top}\H\etav_N^{\textrm{bias}}\big)}\\
&\leq \frac{\lammaxH}{2}\Tr\big(\|\etav_N^{\text{bias}}\|^2\big)\\
&\leq \frac{\lammaxH(1-\gamma\mu)^N}{2\lamminH}\Tr\big(\lamminH\|\etav_0\|^2\big)\\
&\leq \frac{\lammaxH(1-\gamma\mu)^N}{2\lamminH}\bigg(L(\w_0)-L(\ws)\bigg)\qquad\text{(since, $\w_0=\w_0^{\text{bias}}$)}.\\
&\leq \frac{\cnH}{2}\cdot(1-\gamma\mu)^N\bigg(L(\w_0)-L(\ws)\bigg).
\end{align*}
\end{proof}

\subsection{Lemmas for bounding the variance error}\label{sssection:varianceErrorLemmas}
Now, we seek to understand the behavior of the variance error of the tail-averaged iterate $\wbar_{s+1,N}$. We begin by noting here that the variance error is analyzed by beginning the optimization at the solution, i.e. $\etav_0^{\text{variance}}=0$ and allowing the noise to drive the process. In particular, we write out the recursive updates that characterize the variance error:
\begin{align*}
\etav_t^{\textrm{variance}} = \PP_{tb}\etav_{t-1}^{\textrm{variance}} + \gamma \zetav_{tb}, \text{ with } \etav_0^{\textrm{variance}}=0.
\end{align*}
This implies that by defining $\phiv_t^{\textrm{variance}}\defeq\E{\etav_t^{\textrm{variance}}\otimes\etav_t^{\textrm{variance}}}$, we have:
\begin{align}
\label{eq:varOneStep}
\phiv_t^{\textrm{variance}}&=\E{\etav_t^{\textrm{variance}}\otimes\etav_t^{\textrm{variance}}}\nonumber\\
&=\E{\E{\big(\PP_{tb}\etav_{t-1}^{\textrm{variance}} + \gamma \zetav_{tb}\big)\otimes\big(\PP_{tb}\etav_{t-1}^{\textrm{variance}} + \gamma \zetav_{tb}\big)|\mathcal{F}_{t-1}}}\nonumber\\
&=\E{\PP_{tb}\phiv_{t-1}^{\textrm{variance}}\PP_{tb}^{\top}} + \frac{\gamma^2}{b}\Sig.
\end{align}
where, $\mathcal{F}_{t-1}$ is the filtration defined using the samples $\{\x_{ji},y_{ji}\}_{j=1,i=1}^{j=t-1,i=b}$. Furthermore cross terms are zero since $\E{\zetav_{tb}|\mathcal{F}_{t-1}}=0$ owing to first order optimality conditions. Recounting that $\PP_{tb}=\eye-\frac{\gamma}{b}\sum_{i=1}^b\x_{ti}\otimes\x_{ti}$, we express equation~\ref{eq:varOneStep} using a linear operator as follows:
\begin{align*}
\E{\PP_{tb}\phiv_{t-1}^{\textrm{variance}}\PP_{tb}^{\top}}&=\E{\big(\eye-\frac{\gamma}{b}\sum_{i=1}^b\x_{ti}\otimes\x_{ti}\big)\phiv_{t-1}^{\textrm{variance}}\big(\eye-\frac{\gamma}{b}\sum_{i=1}^b\x_{ti}\otimes\x_{ti}\big)}\\
&\defeq (\eyeT-\gamma\Tb)\phiv_{t-1}^{\textrm{variance}},
\end{align*}
with $\Tb$ representing the following linear operator:
\begin{align*}
\Tb = \HL + \HR - \gb \M - \gamma\frac{b-1}{b}\HL\HR,
\end{align*}
with $\M = \E{\x\otimes\x\otimes\x\otimes\x}$, $\HL=\H\otimes\eye$ and $\HR=\eye\otimes\H$ representing the left and right multiplication linear operators corresponding to the matrix $\H$. Given this notation, we consider $\phiv_t^{\textrm{variance}}$:
\begin{align}
\label{eq:phivtvariance}
\phiv_t^{\textrm{variance}}&=(\eyeT-\gamma\Tb)\phiv_{t-1}^{\textrm{variance}} + \frac{\gamma^2}{b}\Sig\nonumber\\
&=\frac{\gamma^2}{b}\bigg(\sum_{k=0}^{t-1} (\eyeT-\gamma\Tb)^k\bigg)\Sig.
\end{align}
Before bounding the variance error, we will describe a lemma that shows that the expected covariance of the variance error $\phiv_{t}^{\textrm{variance}}$ initialized at $0$ grows monotonically to its steady state value (in a PSD sense).
\begin{lemma}\label{lem:psdOrdering}
The sequence of centered variance iterates $\etav_t^{\text{variance}}$ have expected covariances that monotonically grow in a PSD sense, i.e.:
\begin{align*}
0=\phiv_0^{\text{variance}}\preceq\phiv_1^{\text{variance}}\preceq\phiv_2^{\text{variance}}....\preceq\phiv_{\infty}^{\text{variance}}.
\end{align*}
\end{lemma}
\begin{proof}
This lemma generalizes the lemma appearing in~\citet{jain2017markov,JainKKNS17}. We begin by recounting the $t^{\text{th}}$ variance iterate, i.e.:
\begin{align*}
\etav_t^{\text{variance}} = \gamma \sum_{j=1}^{t} \Q_{j+1,t}\zetav_{j,b}.
\end{align*}
This implies in particular that 
\begin{align*}
\phiv_t^{\text{variance}} &= \E{\etav_t^{\text{variance}}\otimes\etav_t^{\text{variance}}}\\ 
&= \gamma^2 \sum_{j=1}^t \sum_{l=1}^t \E{\Q_{j+1,t}\zetav_{j,b}\otimes\zetav_{l,b}\Q_{l+1,b}^{\top}}\quad\text{(from equation~\ref{eq:unrolledRecursion})}\\
&= \gamma^2 \sum_{j=1}^t \sum_{l=1}^t \E{\Q_{j+1,t}\E{\zetav_{j,b}\otimes\zetav_{l,b}|\mathcal{F}_{j-1}}\Q_{l+1,b}^{\top}}\\
&=\gamma^2 \sum_{j=1}^t \E{\Q_{j+1,t}\zetav_{j,b}\otimes\zetav_{j,b}\Q_{j+1,t}^{\top}}\\
&=\frac{\gamma^2}{b}\sum_{j=1}^t\E{\Q_{j+1,t}\Sig\Q_{j+1,t}^{\top}}.
\end{align*}
where, the third line follows since $\E{\zetav_{l,b}\otimes\zetav_{j,b}}=0$ for $j\ne l$, similar to arguments in equation~\ref{eq:perIterCovariance}. This immediately reveals that the sequence of covariances grows as a function of time, since, 
\begin{align*}
\phiv_{t+1}^{\text{variance}}-\phiv_t^{\text{variance}} = \frac{\gamma^2}{b} \E{\Q_{2,t+1}\Sig\Q_{2,t+1}^{\top}}\succeq 0.
\end{align*}
\end{proof}
This lemma leads to a natural upper bound on the variance error, as expressed below:
\begin{lemma}\label{lem:tailAvgVarErr}
With $\gamma<\frac{\gammabmax}{2}$, the variance error of the tail-averaged iterate $\wbar_{s+1,N}^{\text{variance}}$ is upper bounded as:
\begin{align*}
\E{L(\wbar_{s+1,N}^{\text{variance}})}-L(\ws)&\leq\frac{2}{Nb}\Tr{\big(\Tbinv\Sig\big)}.
\end{align*}
\end{lemma}
\begin{proof}
Considering the variance error of tail-averaged iterate from equation~\ref{eq:varianceErrorTailAvg}:
\begin{align*}
\E{L(\wbar_{s+1,N}^{\text{variance}})}-L(\ws)&= \frac{2}{\gamma N^2}\cdot\sum_{l=s+1}^{s+N}\bigg(\Tr{\big(\E{\phiv_l^{\text{variance}}}\big)}\bigg)\\
&\leq \frac{2}{\gamma N}\cdot\Tr{\bigg(\E{\phiv_{\infty}^{\text{variance}}}\bigg)}\qquad\text{(from lemma~\ref{lem:psdOrdering})}\\
&= \frac{2}{\gamma N}\cdot \frac{\gamma^2}{b}\cdot\Tr{\bigg(\sum_{k=0}^{\infty}(\eyeT-\gamma\Tb)^k\Sig\bigg)}\qquad\text{(from equation~\ref{eq:unrolledRecursion})}\\
&=\frac{2}{Nb}\Tr{(\Tbinv\Sig)}.
\end{align*}
\end{proof}

\begin{lemma}\label{lem:finalIterateVarErr}
With $\gamma<\frac{\gammabmax}{2}$, the variance error of the {\bf final} iterate $\w_N^{\text{variance}}$, obtained by running mini-batch SGD for $N$ steps is upper bounded as:
\begin{align*}
\E{L(\w_N^{\text{variance}})}-L(\ws)&\leq\frac{\gamma}{2b}\Tr{\big(\H\Tbinv\Sig\big)}.
\end{align*}
\end{lemma}
\begin{proof}
We note that since we deal with the square loss case, 
\begin{align*}
\E{L(\w_N^{\text{variance}})}-L(\ws)&=\frac{1}{2}\Tr{(\H\phiv_N^{\text{variance}})}\\
&\leq\frac{1}{2}\Tr{(\H\phiv_\infty^{\text{variance}})}\quad\text{(using lemma~\ref{lem:psdOrdering})}\\
&=\frac{\gamma^2}{2b}\Tr{\bigg(\H\sum_{j=0}^{\infty}(\eyeT-\gamma\Tb)^j\Sig\bigg)}\\
&=\frac{\gamma}{2b}\Tr{\bigg(\H\Tbinv\Sig\bigg)}
\end{align*}
\end{proof}

\begin{lemma}\label{lem:useful}
	Denoting the assumption (A) $\gamma \leq \gammabmax/2$,
		\begin{enumerate}
		\item	With (A) in place, $\Tb \succeq 0$. \label{claim:0}
		\item	With (A) in place, $\Tbinv \W \succeq 0$ for every $\W \in \SRd, \; \W \succeq 0$ \label{claim:1}
		\item $\trace{\HLRinv\A}=\frac{1}{2}\trace{\H^{-1}\A}\ \forall\ \A\in\SplRd$ \label{claim:7}
		\item	With (A) in place, \label{claim:8}
		\begin{align*}
			\red{\trace{\Tbinv \Sig} \leq 2\trace{\Hinv \Sig} }
		\end{align*}
	\end{enumerate}
\end{lemma}

\begin{proof}

\emph{Proof of claim~\ref{claim:0} in Lemma~\ref{lem:useful}}:
$\Tb\succeq 0$ implies that for all symmetric matrices $\A\in\SRd$, we have $\trace{\A\Tb\A}\geq 0$, and this is true owing to the following inequalities:
\begin{align*}
	\iprod{\A}{\Tb \A} &= 2 \trace{\A\H\A} - \frac{\gamma}{b} \E{\iprod{\x}{\A \x}^2} - \frac{\gamma(b-1)}{b} \iprod{\H}{\A\H\A} \\
	&\geq 2 \trace{\A\H\A} - \frac{\gamma}{b} \E{\norm{\x}^2 \norm{\A \x}^2} - \frac{\gamma(b-1)}{b} \norm{\H}\trace{\A\H\A} \\
	&\geq 2 \trace{\A\H\A} - \frac{\gamma}{b} \infbound \E{ \norm{\A \x}^2} - \frac{\gamma(b-1)}{b} \norm{\H}\trace{\A\H\A} \\
	&\geq \left(2 - \frac{\gamma}{b} \left(\infbound + {(b-1)} \norm{\H}\right)\right) \trace{\A\H\A},
\end{align*}
and using the definition of $\gammabmax$ finishes the claim.

\emph{Proof of claim~\ref{claim:1} in Lemma~\ref{lem:useful}}:
We require to prove $\Tbinv$ operating on a PSD matrix produces a PSD matrix, or in other words, $\Tbinv$ is a PSD map.
\begin{align}
\label{eq:tbinv}
\Tbinv &= [\HL + \HR - \gb (\M + (b-1) \HL\HR)]^{-1} \nonumber\\
&= (\HL + \HR)^{-\frac{1}{2}}(\HL + \HR)^{\frac{1}{2}}[\HL + \HR - \gb (\M + (b-1) \HL\HR)]^{-1}(\HL + \HR)^{\frac{1}{2}}(\HL + \HR)^{-\frac{1}{2}}\nonumber\\
&= (\HL + \HR)^{-\frac{1}{2}}[\tensor{I} - \gb (\HL + \HR)^{-\frac{1}{2}}(\M + (b-1) \HL\HR)(\HL + \HR)^{-\frac{1}{2}}]^{-1}(\HL + \HR)^{-\frac{1}{2}}
\end{align}
Now, we prove that $\|\gb(\HL + \HR)^{-\frac{1}{2}}(\M + (b-1) \HL\HR)(\HL + \HR)^{-\frac{1}{2}}\|<1$. Given $\gamma<\gammabmax/2$, we employ claim~\ref{claim:0} to note that $\Tb\succ 0$. 
\begin{align*}
\Tb &\succ 0\nonumber\\
&\Rightarrow \HL + \HR -\gb (\M + (b-1) \HL\HR) \succ 0 \nonumber \\
&\Rightarrow \gb (\M + (b-1) \HL\HR) \prec \HL + \HR \nonumber\\
&\Rightarrow \gb (\HL + \HR)^{-\frac{1}{2}}(\M + (b-1) \HL\HR)(\HL + \HR)^{-\frac{1}{2}} \prec \tensor{I}\nonumber\\
&\Rightarrow \|\gb (\HL + \HR)^{-\frac{1}{2}}(\M + (b-1) \HL\HR)(\HL + \HR)^{-\frac{1}{2}} \|< 1
\end{align*}
With this fact in place, we employ Taylor series to expand $\Tbinv$ in equation~\ref{eq:tbinv}, i.e.:
\begin{align*}
\Tbinv &= (\HL + \HR)^{-\frac{1}{2}} \sum_{i=0}^{\infty}( \gb (\HL+\HR)^{-\frac{1}{2}}(\M + (b-1) \HL\HR)(\HL+\HR)^{-\frac{1}{2}} )^{i}(\HL + \HR)^{-\frac{1}{2}}\nonumber\\
&= \sum_{i=0}^{\infty} (\gb (\HL + \HR)^{-1}(\M + (b-1) \HL\HR))^i (\HL+\HR)^{-1}
\end{align*}
The proof completes by employing the following facts: Using Lyapunov's theorem (\cite{bhatia2007} proposition A 1.2.6), we know $(\HL+\HR)^{-1}$ is a PSD map, i.e. if $(\HL+\HR)^{-1}(A) = B$, then, if $A$ is PSD $\implies$ $B$ is PSD. Furthermore, $\M$ is also a PSD map, i.e. if $A_1$ is PSD, $\M(A_1)=E [(x^TA_1x)x\otimes x]$ is PSD as well. Finally, $\HL\HR$ is also a PSD map, since, if $A_2$ is PSD, then, $\HL\HR (A_2)=HA_2H$ which is PSD as well. With all these facts in place, we note that each term in the Taylor's expansion above is a PSD map implying the overall map is PSD as well, thus rounding up the proof to claim~\ref{claim:1} in Lemma~\ref{lem:useful}.

\emph{Proof of claim~\ref{claim:7} in Lemma~\ref{lem:useful}}:

We know that the operator $\HLRinv$ is a PSD map, i.e, it maps PSD matrices to PSD matrices. Since $\A\succeq 0$, we replace this condition with $\U=\HLRinv\A\succeq\ 0$ implying, we need to show the following:
\begin{align*}
\trace{\U} &= \frac{1}{2} \trace{\H^{-1}\red{\A}}\ \forall\ \U\succeq 0
\end{align*}
Examining the right hand side, we see the following:
\begin{align*}
\frac{1}{2}\trace{\H^{-1}\red{\A}} &= \frac{1}{2}\trace{\H^{-1} (\HL + \HR)\U}\nonumber\\
&= \frac{1}{2}\trace{\H^{-1}\H\U + \H^{-1}\U\H}\nonumber\\
&=\trace{\U}
\end{align*}
thus wrapping up the proof of claim~\ref{claim:8}.

\emph{Proof of claim~\ref{claim:8} in Lemma~\ref{lem:useful}}:
Let $\UT = \HL + \HR - \gb \cdot (b-1) \HL\HR$. Then, 
\begin{align*}
\Tbinv\Sig &= \bigg(\UT - \gb \M\bigg)^{-1}\Sig\\
&= \sum_{i=0}^{\infty} \bigg(\gb \UT^{-1}\M\bigg)^{i} \UT^{-1}\Sig.
\end{align*}
Let $\A=\UT^{-1}\Sig$, $\A'=\UT^{-1}\H$. Then,
\begin{align*}
\Tbinv\Sig = \sum_{i=1}^{\infty} \bigg(\gb \UT^{-1}\M\bigg)^i\A.
\end{align*}
The $i=0$ term is just $=\A$. Now, considering $i=1$, we have:
\begin{align*}
\gb \UT^{-1}\M \A &\preceq \gb\|\A\|_2 \UT^{-1}\M\eye\\
&\preceq \gb \|\A\|_2\infbound \UT^{-1}\H=\gb \|\A\|_2\infbound \A'.
\end{align*}
Next, considering $i=2$, we have:
\begin{align*}
\bigg(\gb \UT^{-1}\M\bigg)^2 \A &= \bigg(\gb \UT^{-1}\M\bigg) \cdot \bigg(\gb \UT^{-1}\M\bigg) \A\\
&\preceq \bigg(\gb \|\A\|_2\infbound\bigg) \cdot \bigg(\gb \UT^{-1}\M\bigg)\A'\\
&\preceq \bigg(\gb \|\A\|_2\infbound\bigg) \cdot \bigg(\gb \UT^{-1}\bigg) \cdot \|\A'\|_2\cdot \infbound\H\\
&\preceq \bigg(\gb \|\A\|_2\infbound\bigg) \cdot \bigg(\gb \|\A'\|_2\infbound\bigg) \cdot \A'.
\end{align*}
Noting this recursive structure, we see that:
\begin{align*}
\Tbinv\Sig&=\sum_{i=0}^\infty \bigg(\gb \UT^{-1}\M\bigg)^i \A\\
&\preceq \A + \sum_{i=1}^{\infty} \bigg(\gb \|\A\|_2\infbound\bigg) \cdot \bigg(\gb \|\A'\|_2\infbound\bigg)^{i-1}\cdot \A'\\
&= \A + \frac{\bigg(\gb \|\A\|_2\infbound\bigg)}{1-\bigg(\gb \|\A'\|_2\infbound\bigg)}\cdot\A'.
\end{align*}
Note that this summation is finite iff $\gamma\leq \frac{b}{\infbound\|\A'\|_2}$. Further, applying the trace operator on both sides, we have:
\begin{align}
\label{eq:tbinvSig}
\trace{\Tbinv\Sig}\leq \trace{\A} + \frac{\bigg(\gb \|\A\|_2\infbound\bigg)}{1-\bigg(\gb \|\A'\|_2\infbound\bigg)}\trace{\A'}.
\end{align}
Now, for any psd matrix $\B\succeq0$, let us upperbound $\UT^{-1}\B$:
\begin{align*}
\UT^{-1}\B&=\sum_{j=0}^{\infty} \bigg( \gamma\cdot \frac{b-1}{b}\cdot(\HL+\HR)^{-1}\cdot\HL\HR\bigg)^i (\HL+\HR)^{-1}\Sig.
\end{align*}
The recursion can be bounded by analyzing $i=1$:
\begin{align*}
&\gamma\cdot \frac{b-1}{b}\cdot(\HL+\HR)^{-1}\cdot\HL\HR\cdot(\HL+\HR)^{-1}\B\\
&\preceq\|(\HL+\HR)^{-1}\B\|_2\cdot\gamma\cdot\frac{b-1}{b}\cdot(\HL+\HR)^{-1}\cdot\HL\HR\cdot\eye\\
&\preceq \|(\HL+\HR)^{-1}\B\|_2\cdot\gamma\cdot\frac{b-1}{b}\cdot(\HL+\HR)^{-1}\cdot \|\H\|_2\H\\
&=\|(\HL+\HR)^{-1}\B\|_2\cdot\gamma\cdot\frac{b-1}{2b}\cdot\|\H\|_2\cdot\eye
\end{align*}
This indicates the means to recurse for bounding terms $i\geq 2$:
\begin{align*}
\UT^{-1}\B&\preceq\sum_{j=0}^\infty \|(\HL+\HR)^{-1}\B\|_2 \bigg(\gamma \cdot \frac{b-1}{2b}\cdot\|\H\|_2\bigg)^j\cdot\eye\\
&=\frac{\|(\HL+\HR)^{-1}\B\|_2}{1-\gamma\cdot\frac{(b-1)\|\H\|_2}{2b}}\cdot \eye.
\end{align*}
The upperbound above is true as long as $\gamma<\frac{2b}{(b-1)\|\H\|_2}$. This now allows us to obtain bounds on $\|\A\|_2,\|\A'\|_2,\trace{\A'}$:
\begin{align*}
\|\A\|_2\leq\frac{\|(\HL+\HR)^{-1}\Sig\|_2}{1-\gamma\cdot\frac{b-1}{2b}\cdot\|\H\|_2}\\
\|\A'\|_2\leq\frac{1/2}{1-\gamma\cdot\frac{b-1}{2b}\cdot\|\H\|_2}\\
\trace{\A'}\leq\frac{d/2}{1-\gamma\cdot\frac{b-1}{2b}\cdot\|\H\|_2}.
\end{align*}
Substituting these in equation~\ref{eq:tbinvSig}:
\begin{align}
\label{eq:tbinvSig1}
\trace{\Tbinv\Sig}\leq\trace{\A} + \frac{\frac{\gamma\infbound}{2b}\cdot d\|(\HL+\HR)^{-1}\Sig\|_2}{\bigg(1-\frac{\gamma}{2b}\cdot(\infbound+(b-1)\|\H\|_2)\bigg)\cdot\bigg(1-\gamma\cdot\frac{b-1}{2b}\|\H\|_2\bigg)}.
\end{align}
with the conditions on $\gamma$ being: $\gamma\leq\frac{2b}{(b-1)\|\H\|_2}$, $\gamma\leq\frac{2b}{\infbound + (b-1)\|\H\|_2}$, $\gamma\leq \frac{2b}{\infbound}$. These are combined using $\gamma\leq\frac{2b}{\infbound+(b-1)\|\H\|_2}$. Once this condition is satisfied, the denominator of the second term can be upperbounded by atmost a constant. Next, looking at the numerator of the second term, we see that $\gamma\leq\frac{2b}{\infbound\cdot\frac{d\|(\HL+\HR)^{-1}\Sig\|_2}{\trace{(\HL+\HR)^{-1}\Sig}}}=\frac{2b}{\infbound\rhom}$ allows for the second term to be upperbounded by $\mathcal{O}(\trace{(\HL+\HR)^{-1}\Sig})$. This is clearly satisfied if $\gamma\leq\frac{2b}{\infbound \cdot\rhom + (b-1)\|\H\|_2}$. In particular, setting $\gamma$ to be half of this maximum, we have:
\begin{align}
\label{eq:tbinvSig2}
\trace{\Tbinv\Sig} \leq \trace{\A} + 2\trace{(\HL+\HR)^{-1}\Sig}.
\end{align}
Next, for obtaining a sharp bound on $\trace{\A}$, we require comparing $\trace{\hat{\Sig}}=\trace{(\HL+\HR-\gamma\cdot\frac{b-1}{b}\cdot\HL\HR)^{-1}\Sig}$ with $\trace{\tilde{\Sig}}=\trace{(\HL+\HR)^{-1}\Sig}$. For this, without loss of generality, we can consider $\H$ to be diagonal, and this implies that comparing the diagonal elements of $\hat{\Sig}_{ii} = \Sig_{ii}/(2\lambda_i-\gamma\frac{b-1}{b}\lambda_i^2)$ while $\tilde{\Sig}_{ii} = \Sig_{ii}/2\lambda_i$. Comparing these, we see that 
\begin{align*}
\trace{\hat{\Sig}}=\trace{(\HL+\HR-\gamma\cdot\frac{b-1}{b}\cdot\HL\HR)^{-1}\Sig}&\leq\frac{1}{1-\gamma\frac{b-1}{2b}\|\H\|_2}\trace{\tilde{\Sig}}\\
&=\frac{1}{1-\gamma\frac{b-1}{2b}\|\H\|_2}\trace{(\HL+\HR)^{-1}\Sig}.
\end{align*}
Noting that $\trace{\A}=\trace{\hat{\Sig}}$, we see that substituting the above in equation~\ref{eq:tbinvSig2}, we have:
\begin{align*}
\trace{\Tbinv\Sig} &\leq \frac{1}{1-\gamma\frac{b-1}{2b}\|\H\|_2}\trace{(\HL+\HR)^{-1}\Sig} + 2 \trace{(\HL+\HR)^{-1}\Sig}\\
&\leq 4 \trace{(\HL+\HR)^{-1}\Sig}= 2\trace{\Hinv\Sig}.
\end{align*}
\end{proof}

\begin{corollary}\label{cor:tailAvgVar}
Consider the agnostic case of the streaming LSR problem. With $\gamma\leq\frac{\gammabmax}{2}$, the variance error of the tail-averaged iterate $\wbar_{s+1,N}^{\text{variance}}$ is upper bounded as:
\begin{align*}
\E{L(\wbar_{s+1,N}^{\text{variance}})}-L(\ws)&\leq\frac{4}{Nb}\cdot\widehat{\sigma^2_{\text{MLE}}}.
\end{align*}
\end{corollary}
\begin{proof}
The result follows in a straightforward manner by noting that $\gamma\leq\frac{\gammabmax}{2}$ implying that $\trace{\Tbinv\Sig}\leq2\trace{\Hinv\Sig}$ and by substituting into the result of lemma~\ref{lem:tailAvgVarErr}. 
\end{proof}

\begin{corollary}\label{cor:finalIterVar}
With $\gamma\leq\frac{\gammabmax}{2}$, $\Sig=\sigma^2\H$ the variance error of the {\bf final} iterate $\w_N^{\text{variance}}$, obtained by running mini-batch SGD for $N$ steps is upper bounded as:
\begin{align*}
\E{L(\w_N^{\text{variance}})}-L(\ws)&\leq\frac{\gamma\sigma^2}{2b}\Tr{\H}.
\end{align*}
\end{corollary}
\begin{proof}
This follows from the fact that $\Tbinv\Sig\preceq\sigma^2\eye$, implying that $\H\Tbinv\Sig\preceq\sigma^2\H$ and then applying the trace operator on the result of lemma~\ref{lem:finalIterateVarErr}. 
\end{proof}

\subsection{Main Results}\label{ssec:mainRes}
\subsubsection{Derivation of Divergent Learning Rate}
\label{ssec:divLearnRateDerivation}
A necessary condition for the convergence of Stochastic Gradient Updates is $\Tb\succeq0$, and this by definition implies,
\begin{align*}
&\iprod{\W}{\Tb\W}\geq0,\ \ \W\in\SRd\\
\implies &2\Tr{(\W\H\W)} - \gb \Tr{(\W\M\W)}-\gamma\big(\frac{b-1}{b}\big)\Tr{(\W\H\W\H)}\geq0\\
\implies &\frac{2}{\gamma}\geq\frac{\Tr{(\W\M\W)}+(b-1)\Tr{(\W\H\W\H)}}{b\Tr{(\W\H\W)}}\\
\implies &\frac{2}{\gammabmaxdiv}=\sup_{\W\in\SRd}\frac{\Tr{(\W\M\W)}+(b-1)\Tr{(\W\H\W\H)}}{b\Tr{(\W\H\W)}}.
\end{align*}

\subsubsection{Proof of Theorem~\ref{thm:mainMBTANR}}
\label{sssection:ptt1}
\begin{proof}[proof of Theorem~\ref{thm:mainMBTANR}]
The proof of theorem~\ref{thm:mainMBTANR} follows from characterizing bias-variance decomposition for the tail-averaged iterate in section~\ref{sssection:biasVarDecompTailAvgIterate} with equation~\ref{eq:genErrorNewUpperBound}.

The bias error of the tail-averaged iterate (equation~\ref{eq:biasErrorTailAvg}) is bounded with lemma~\ref{lem:biasContract} and lemma~\ref{lem:tailAveragedBiasError} in section~\ref{sssection:biasErrorLemmas}.

The variance error of the tail-averaged iterate (equation~\ref{eq:varianceErrorTailAvg}) is bounded with lemma~\ref{lem:psdOrdering}, lemma~\ref{lem:tailAvgVarErr}, lemma~\ref{lem:useful} and corollary~\ref{cor:tailAvgVar} in section~\ref{sssection:varianceErrorLemmas}.

The final expression follows through substituting the result of lemma~\ref{lem:tailAveragedBiasError} and corollary~\ref{cor:tailAvgVar} into equation~\ref{eq:genErrorNewUpperBound}, with appropriate parameters of the problem, i.e., with a batch size $b$, number of burn-in iterations $s$, number of tail-averaged iterations $n/b-s$ to provide the claimed excess risk bound of Algorithm~\ref{alg:mbSGD}:
\begin{align*}
\E{L(\wbar)}-L(\ws)\leq \frac{2}{\gamma^2 \mu^2 (\frac{n}{b}-s)^2}\cdot (1-\gamma\mu)^s\cdot\bigg(L(\w_0)-L(\ws)\bigg) + 4\cdot\frac{\widehat{\sigma^2_{\text{MLE}}}}{b\cdot(\frac{n}{b}-s)}.
\end{align*}
\end{proof}

\subsubsection{Proof of Lemma~\ref{lem:lastPointLemma}}
\label{sssection:ptl2}
\begin{proof}[proof of Lemma~\ref{lem:lastPointLemma}]
The proof of lemma~\ref{lem:lastPointLemma} follows from characterizing bias-variance decomposition for the final iterate in section~\ref{sssection:biasVarDecompFinalIterate} with equation~\ref{eq:lastPointBVDecomposition}.

The bias error of the final iterate is bounded with lemma~\ref{lem:biasContract} and lemma~\ref{lem:lastPointBiasErrorBound} in section~\ref{sssection:biasErrorLemmas}.

The variance error of the final iterate is bounded with lemma~\ref{lem:psdOrdering}, lemma~\ref{lem:finalIterateVarErr}, lemma~\ref{lem:useful} and corollary~\ref{cor:finalIterVar} in section~\ref{sssection:varianceErrorLemmas}.

The final expression follows through substituting the result of lemma~\ref{lem:lastPointBiasErrorBound} and corollary~\ref{cor:finalIterVar} into equation~\ref{eq:lastPointBVDecomposition}, with appropriate parameters of the problem, i.e., with a batch size $b$, number of samples $n$ and number of iterations $\lfloor n/b\rfloor$, to provide the claimed excess risk bound:
\begin{align*}
\E{L(\w_{\lfloor n/b\rfloor})}-L(\ws)\leq \cnH_b (1-\gamma\mu)^{\lfloor n/b \rfloor}\bigg(L(\w_0)-L(\ws)\bigg) + \gb \sigma^2\Tr{(\H)},
\end{align*}
\end{proof}

\subsubsection{Proof of Theorem~\ref{thm:DOminibatchSGD}}
\label{sssection:ptt3}
\begin{proof}
Let $\widetilde{L_\ind}=\E{L(\w_\ind)}-L(\ws)$. We will first provide a recursive bound for $\widetilde{L_\ind}$ for $\ind\leq\log\left(\frac{n}{bt}\right)-1$ using theorem~\ref{thm:mainMBTANR}, with a mini-batch size of $b_\ind = 1+2^{\ind-1}b$, where, $b =\bthresh-1 $, $n_\ind = b_\ind \cdot t$, $s=t-1$:
\begin{align*}
\widetilde{L_\ind}&\leq 2\cnH_{b_e}^2\exp\bigg(-\frac{n_e}{b_e\cdot\cnH_{b_e}}\bigg)\widetilde{L_{\ind-1}} + 4\frac{\widehat{\sigma^2_{\text{MLE}}}}{b_\ind}\\
&\leq \exp\bigg(-\frac{n_\ind}{3b_\ind\cnH_\ind\log(\cnH_\ind)}\bigg)\cdot\widetilde{L_{\ind-1}} + 4\cdot\frac{\widehat{\sigma^2_{\text{MLE}}}}{b_\ind}.
\end{align*}
Next, denote $\cnH=\twonorm{\H}/\mu$; now, let us bound $\cnH_{b_\ind}$:
\begin{align*}
\cnH_{b_\ind} &= \frac{\infbound \cdot\frac{d\twonorm{(\HL+\HR)^{-1}\Sig}}{\trace{(\HL+\HR)^{-1}\Sig}} + (b_\ind-1)\twonorm{\H}}{b_\ind\mu}\\
&= \cnH \cdot \frac{\bthresh-1 + b_\ind-1}{b_\ind}= \cnH \cdot \frac{\bthresh-1 + 2^{\ind-1}(\bthresh-1)}{2^{\ind-1}(\bthresh-1)}\\
&=\cnH\cdot\frac{1+2^{\ind-1}}{2^{\ind-1}}\leq 2\cnH.
\end{align*}
This implies $\cnH_{b_\ind}\log(\cnH_{b_\ind})\leq4\cnH\log(\cnH)$. This implies, revisiting the recursion on $\widetilde{L_\ind}$, we have:
\begin{align}
\label{eq:pt1}
\widetilde{L_\ind}&\leq\exp\bigg(-\frac{n_\ind}{12b_\ind\cnH\log(\cnH)}\bigg)\cdot\widetilde{L_{\ind-1}} + 4\cdot\frac{\widehat{\sigma^2_{\text{MLE}}}}{b_\ind}\nonumber\\
&\leq\exp\bigg(-\frac{t}{12\cnH\log(\cnH)}\bigg)\cdot\widetilde{L_{\ind-1}} + 4\cdot\frac{\widehat{\sigma^2_{\text{MLE}}}}{2^{\ind-1}b}\nonumber\\
&\leq \exp\bigg(-\frac{t\ind}{12\cnH\log(\cnH)}\bigg)\cdot\widetilde{L_{0}} + \frac{4\widehat{\sigma^2_{\text{MLE}}}}{b}\cdot\sum_{j=1}^\ind \frac{\exp\bigg(-\frac{t(j-1)}{12\cnH\log(\cnH)}\bigg)}{2^{\ind-j}}\nonumber\\
&\leq \exp\bigg(-\frac{t\ind}{12\cnH\log(\cnH)}\bigg)\cdot\widetilde{L_{0}} + \frac{4\widehat{\sigma^2_{\text{MLE}}}}{b}\cdot \frac{1/2^{\ind-1}}{1-2\cdot\exp\big(-\frac{t}{12\cnH\log\cnH}\big)}\nonumber\\
&\leq \exp\bigg(-\frac{t\ind}{12\cnH\log(\cnH)}\bigg)\cdot\widetilde{L_{0}} + \frac{12\widehat{\sigma^2_{\text{MLE}}}}{2^\ind b}\quad\text{(since $t>24\cnH\log(\cnH)$)}\nonumber\\
&=\exp\bigg(-\frac{t\ind}{12\cnH\log(\cnH)}\bigg)\cdot\widetilde{L_{0}} + \frac{12\widehat{\sigma^2_{\text{MLE}}}}{b\cdot n}\cdot(4bt)\quad\text{(since $2^\ind=n/(4bt)$)}\nonumber\\
&=\exp\bigg(-\frac{t\ind}{12\cnH\log(\cnH)}\bigg)\cdot\widetilde{L_{0}} + 48\cdot\frac{\widehat{\sigma^2_{\text{MLE}}}t}{n}
\end{align}
Next, for the final epoch, we have $b=n/2t$, $s=t/2$, and a total of $n/2$ samples, implying:
\begin{align}
\widetilde{L_{\ind+1}} &\leq \frac{2\cnH_b^2}{\bigg(t/2\bigg)^2}\cdot\exp\big(-\frac{t}{2\cnH_b}\big)\widetilde{L_\ind} + 4\cdot\frac{\widehat{\sigma^2_{\text{MLE}}}}{b\cdot\big(n/4b\big)}= \frac{8\cnH_b^2}{t^2}\cdot\exp\big(-\frac{t}{2\cnH_b}\big)\cdot\widetilde{L_\ind} + 16\cdot\frac{\widehat{\sigma^2_{\text{MLE}}}}{n}\nonumber\\
&\leq \frac{32\cnH^2}{t^2}\cdot\exp\big(-\frac{t}{4\cnH}\big)\cdot\widetilde{L_\ind} + 16\cdot\frac{\widehat{\sigma^2_{\text{MLE}}}}{n}\quad\text{(since $\cnH_b\leq2\cnH$)}\nonumber\\
&\leq \frac{32\cnH^2}{t^2}\cdot\exp\big(-\frac{t}{4\cnH}\big)\cdot\bigg(\exp\bigg(-\frac{t\ind}{12\cnH\log(\cnH)}\bigg)\cdot\widetilde{L_{0}} + 48\cdot\frac{\widehat{\sigma^2_{\text{MLE}}}t}{n}\bigg) + 16\cdot\frac{\widehat{\sigma^2_{\text{MLE}}}}{n}\nonumber\\
&\leq \frac{32\cnH^2}{t^2}\cdot\exp\big(-\frac{t}{4\cnH}\big)\cdot\exp\bigg(-\frac{t\ind}{12\cnH\log(\cnH)}\bigg)\cdot\widetilde{L_{0}} + 64\cnH\exp\big(-t/4\cnH\big)\cdot\frac{\widehat{\sigma^2_{\text{MLE}}}}{n} + 16\cdot\frac{\widehat{\sigma^2_{\text{MLE}}}}{n}\nonumber\\
&\leq \frac{32\cnH^2}{t^2}\cdot\exp\big(-\frac{t}{4\cnH}\big)\cdot\exp\bigg(-\frac{t\ind}{12\cnH\log(\cnH)}\bigg)\cdot\widetilde{L_{0}} + 80 \frac{\widehat{\sigma^2_{\text{MLE}}}}{n}\nonumber\\
&\leq \exp\bigg(-\frac{t(\ind+1)}{12\cnH\log(\cnH)}\bigg)\cdot\widetilde{L_{0}} + 80 \frac{\widehat{\sigma^2_{\text{MLE}}}}{n}\nonumber\\
&= \bigg(\frac{2bt}{n}\bigg)^{\frac{t}{12\cnH\log(\cnH)}} \widetilde{L_0}+ 80\cdot\frac{\widehat{\sigma^2_{\text{MLE}}}}{n},
\end{align}
which rounds up the proof of the theorem.
\end{proof}
\subsubsection{Proof of Theorem~\ref{thm:parameterMixing}}
\label{ssection:parameterMixing}
\begin{proof}
For analyzing the parameter mixing scheme, we require tracking the progress of the $i^{th}$ machine's SGD updates using its centered estimate $\etav_k^{(i)}$. Furthermore, the tail-averaged iterate for the $i^{th}$ machine is representeed as $\etavb^{(i)}\defeq\frac{1}{N}\sum_{k=s+1}^{s+N}\etav_k^{(i)}$. Finally, the model averaged estimate is represented with its own centered estimate defined as $\etavb=\frac{1}{P}\sum_{i=1}^P \etavb^{(i)}$. Now, in a manner similar to standard mini-batch tail-averaged SGD on a single machine, the model averaged iterate admits its own bias variance decomposition, through which $\etavb = \etavb^{\textrm{bias}} + \etavb^{\textrm{variance}}$ and an upperbound on the excess risk is written as:
\begin{align*}
\E{L(\wbar)}-L(\ws)&=\E{\frac{1}{2}\iprod{(\wbar-\ws)}{\H(\wbar-\ws)}}=\E{\frac{1}{2}\iprod{\etavb}{\H\etavb}}\\
&\leq\E{\iprod{\etavb^{\text{bias}}}{\H\etavb^{\text{bias}}}} + \E{\iprod{\etavb^{\text{variance}}}{\H\etavb^{\text{variance}}}}.
\end{align*}
We will first handle the variance since it is straightforward given that the noise $\zetav$ is independent for different machines SGD runs. What this implies is the following:
\begin{align}
\label{eq:varErrorPM}
\etavb^{\text{variance}}&=\frac{1}{P}\sum_{i=1}^P\etavb^{(i),\text{variance}}\nonumber\\
\implies \E{\etavb^{\text{variance}}\otimes\etavb^{\text{variance}}}&=\frac{1}{P^2}\sum_{i,j}\E{\etavb^{(i),\text{variance}}\otimes\etavb^{(j),\text{variance}}}\nonumber\\
&=\frac{1}{P^2}\bigg(\sum_i \E{\etavb^{(i),\text{variance}}\otimes\etavb^{(i),\text{variance}}} + \sum_{i\ne j}\E{\etavb^{(i),\text{variance}}\otimes\etavb^{(j),\text{variance}}}\bigg)\nonumber\\
&=\frac{1}{P}\E{\etavb^{(1),\text{variance}}\otimes\etavb^{(1),\text{variance}}}.
\end{align}
Where, the final line follows because $\forall\ i\ne j$, the terms are in expectation equal to zero since in expectation each of the noise terms is zero (from first order optimality conditions). The other observation is that the only terms left are $P$ independent runs of tail-averaged SGD in each of the machine, whose risk is straightforward to bound from corollary~\ref{cor:tailAvgVar}. This implies
\begin{align}
\label{eq:varErrPM}
\iprod{\H}{\E{\etavb^{\text{variance}}\otimes\etavb^{\text{variance}}}}&\leq\frac{4}{PNb}\cdot\widehat{\sigma^2_{\text{MLE}}}.\quad\text{(using corollary~\ref{cor:tailAvgVar})}
\end{align}
Next, let us consider the bias error: 
\begin{align}
\label{eq:biasErrorPM1}
\etavb^{\text{bias}}&=\frac{1}{P}\sum_{i=1}^P\etavb^{(i),\text{bias}}\nonumber\\
\implies \E{\etavb^{\text{bias}}\otimes\etavb^{\text{bias}}}&=\frac{1}{P^2}\sum_{i,j}\E{\etavb^{(i),\text{bias}}\otimes\etavb^{(j),\text{bias}}}\nonumber\\&=\frac{1}{P^2}\bigg(\sum_{i=1}^P\underbrace{\E{\etavb^{(i),\text{bias}}\otimes\etavb^{(i),\text{bias}}}}_{\text{independent runs of tail-averaged SGD}} + \sum_{i\ne j}\E{\etavb^{(i),\text{bias}}\otimes\etavb^{(j),\text{bias}}}\bigg),
\end{align}
which implies that we require bounding $\forall\ i\ne j$, $\E{\etavb^{(i),\text{bias}}\otimes\etavb^{(j),\text{bias}}}$.
\begin{align*}
\E{\etavb^{(i),\text{bias}}\otimes\etavb^{(j),\text{bias}}} &= \frac{1}{N^2}\sum_{k,l=s+1}^{s+N}\E{\etav^{(i),\text{bias}}_k\otimes\etav^{(j),\text{bias}}_l}=\frac{1}{N^2}\sum_{k,l=s+1}^{s+N}\E{\etav^{(i),\text{bias}}_k}\otimes\E{\etav^{(j),\text{bias}}_l}\\
&=\frac{1}{N^2}\sum_{k,l=s+1}^{s+N}\E{\Q_{1:k}^{(i)}\etav_0}\otimes\E{\Q_{1:l}^{(j)}\etav_0}\quad\text{(from equation~\ref{eq:biasRec})}\\
&=\frac{1}{N^2}\bigg(\sum_{k=s+1}^{s+N}(\eye-\gamma\H)^k\bigg)\etav_0\otimes\etav_0\bigg(\sum_{l=s+1}^{s+N}(\eye-\gamma\H)^l\bigg)\\
&\preceq\frac{1}{N^2}\bigg(\sum_{k=s+1}^{\infty}(\eye-\gamma\H)^k\bigg)\etav_0\otimes\etav_0\bigg(\sum_{l=s+1}^{\infty}(\eye-\gamma\H)^l\bigg)\\
&=\frac{1}{\gamma^2N^2}\Hinv(\eye-\gamma\H)^{s+1}\etav_0\otimes\etav_0(\eye-\gamma\H)^{s+1}\Hinv.
\end{align*}
This implies that,
\begin{align}
\label{eq:crossTermBiasPM}
\E{\etavb^{(i),\text{bias}}\otimes\etavb^{(j),\text{bias}}} &\leq \frac{1}{\gamma^2 N^2}\cdot\iprod{\H}{\Hinv(\eye-\gamma\H)^{s+1}\etav_0\otimes\etav_0(\eye-\gamma\H)^{s+1}\Hinv}\nonumber\\
&= \frac{1}{\gamma^2N^2}\cdot\etav_0^\top (\eye-\gamma\H)^{s+1}\Hinv\H\Hinv(\eye-\gamma\H)^{s+1}\etav_0\nonumber\\
&\leq \frac{(1-\gamma\mu)^{2s+2}}{\mu\gamma^2N^2}\|\etav_0\|^2\leq\frac{(1-\gamma\mu)^{2s+2}}{\mu^2\gamma^2N^2}\cdot\bigg(L(\w_0)-L(\ws)\bigg).
\end{align}
Combining the bound for the cross terms in equation~\ref{eq:crossTermBiasPM} and lemma~\ref{lem:tailAveragedBiasError} for the self-terms, we get:
\begin{align}
\label{eq:biasErrPM}
\iprod{\H}{\E{\etavb^{\text{bias}}\otimes\etavb^{\text{bias}}}}\leq \frac{(1-\gamma\mu)^{s+1}}{\mu^2\gamma^2N^2}\cdot\frac{2+(1-\gamma\mu)^{s+1}\cdot(P-1)}{P}\cdot\bigg(L(\w_0)-L(\ws)\bigg).
\end{align}
The proof wraps up by substituting the relation $N=n/(P\cdot b)-s$ in equations~\ref{eq:varErrPM} and~\ref{eq:biasErrPM}.
\end{proof}
\vspace*{-5mm}
\subsubsection{Proof of Lemma~\ref{lem:agnosticLowerBound1}}
\label{ssection:sep}
For this problem instance, we begin by noting that $(\HL+\HR)^{-1}\Sig$ is diagonal as well, with entries:
\[ \{(\HL+\HR)^{-1}\Sig\}_{ii} = \frac{1}{2} \{\Hinv\Sig\}_{ii}=
  \begin{cases}
    1/2       & \quad \mathrm{if} \ i=1\\
    1/2(d-1) & \quad \mathrm{if} \ i>1\\
  \end{cases}
\]
Let us consider the case with batch size $b=1$. With the appropriate choice of step size $\gamma$ that ensure contracting operators, we require considering $\trace{\Tbinv\Sig}$ as in equation~\ref{eq:tbinvSig}, which corresponds to bounding the leading order term in the variance. We employ the taylor's expansion (just as in claim~\ref{claim:1} of lemma~\ref{lem:useful}) to expand the term of interest $\Tbinv\Sig$:
\begin{align*}
\Tbinv\Sig &= \sum_{i=0}^{\infty} \left(\gamma (\HL + \HR)^{-1}\M\right)^i (\HL+\HR)^{-1}\Sig\\
&=(\HL+\HR)^{-1}\Sig+\sum_{i=1}^{\infty} \left(\gamma (\HL + \HR)^{-1}\M\right)^i (\HL+\HR)^{-1}\Sig\\
\Rightarrow \Tr\Tbinv\Sig&=\Tr(\HL+\HR)^{-1}\Sig+\sum_{i=1}^{\infty} \Tr\left[\left(\gamma (\HL + \HR)^{-1}\M\right)^i (\HL+\HR)^{-1}\Sig\right]\\
\Tr\Tbinv\Sig&=\frac{1}{2}\Tr\Hinv\Sig+\sum_{i=1}^{\infty} \Tr\left[\left(\gamma (\HL + \HR)^{-1}\M\right)^i (\HL+\HR)^{-1}\Sig\right]
\end{align*}
We observe that the term corresponding to $i=0$ works out regardless of the choice of stepsize $\gamma$; we then switch our attention to the second term, i.e., the term corresponding to $i=1$:
\begin{align*}
\Tr\left(\gamma (\HL + \HR)^{-1}\M\right) (\HL+\HR)^{-1}\Sig &= \frac{d+2}{4}\cdot\trace{\Sig}
\end{align*}
We require that this term should be $\leq \Tr(\HL+\HR)^{-1}\Sig$, implying,
\begin{align*}
\gamma<\frac{4\Tr(\HL+\HR)^{-1}\Sig}{(d+2)\trace{\Sig}}
\end{align*}
For this example, we observe that this yields $\gamma<\frac{4}{(d+2)(1+\frac{1}{d})}$, which clearly is off by a factor $d$ compared to the well-specified case which requires $\gamma<\frac{d}{(d+2)(1+\frac{1}{d})}$, thus indicating a clear separation between the step sizes required by SGD for the well-specified and mis-specified cases.
\subsubsection{Proofs of supporting lemmas}
\label{sssection:suppLemmasProof}

\underline{\bf Proof of lemma~\ref{lem:biasMBSGD}}
\begin{proof}[Proof of lemma~\ref{lem:biasMBSGD}]
We begin by considering $\iprod{\eye}{\E{\etav_t^{\text{bias}}\otimes\etav_t^{\text{bias}}}}$:
\begin{align*}
\iprod{\eye}{\E{\etav_t^{\text{bias}}\otimes\etav_t^{\text{bias}}}} &= \E{\|\etav_t^{\textrm{bias}}\|^2}\\
&=\E{ (\etav_{t-1}^{\textrm{bias}})^\top  \bigg(\eye-\gb\sum_{i=1}^b \x_{ti}\x_{ti}^\top\bigg) \bigg(\eye-\gb\sum_{i=1}^b \x_{ti}\x_{ti}^\top\bigg) \etav_{t-1}^{\textrm{bias}}  }\\
&\leq (1-\gamma\mu)\cdot\E{\|\etav_{t-1}^{\textrm{bias}}\|^2}\qquad\text{(from lemma~\ref{lem:biasContract})},
\end{align*}
from where the lemma follows through substitution of $\gamma = \gammabmax/2$.
\end{proof}

\underline{\bf Proof of lemma~\ref{lem:varMBSGD}}
\begin{proof}[Proof of lemma~\ref{lem:varMBSGD}]
From equation~\ref{eq:phivtvariance}, we have that:
\begin{align*}
\phiv_t^{\textrm{variance}}&=\E{\etav_t^{\textrm{variance}}\otimes\etav_t^{\textrm{variance}}}\\
&=\frac{\gamma^2}{b}\bigg(\sum_{k=0}^{t-1} (\eyeT-\gamma\Tb)^k\bigg)\Sig
\end{align*}
Allowing $t\to\infty$, we have:
\begin{align*}
\phiv_\infty^{\textrm{variance}} = \frac{\gamma}{b}\Tbinv\Sig\preceq\gb\cdot\sigma^2\eye\quad\text{(from claim~\ref{claim:8} in lemma~\ref{lem:useful} since $\gamma\leq\gammabmax/2$, $\Sig=\sigma^2\H$)}.
\end{align*}
Substituting $\gamma=\gammabmax/2$, the result follows.
\end{proof}

\end{document}